\documentclass{article}

\pdfoutput=1

\usepackage{EMNLP2022}

\usepackage{times}
\usepackage{latexsym}

\usepackage[T1]{fontenc}

\usepackage[utf8]{inputenc}

\usepackage{microtype}

\usepackage{inconsolata}

\usepackage{xcolor}
\usepackage{multirow}
\usepackage[pdftex]{graphicx}
\usepackage{amsmath}
\usepackage{amssymb}
\usepackage{amsthm}
\usepackage[ruled,vlined,linesnumbered]{algorithm2e}
\usepackage{enumitem}
\usepackage{booktabs}
\usepackage{listings}

\theoremstyle{plain}
\newtheorem{theorem}{Theorem}[section]

\newtheorem{lemma}[theorem]{Lemma}

\theoremstyle{definition}
\newtheorem{definition}[theorem]{Definition}

\theoremstyle{remark}

\usepackage{pifont}
\newcommand{\cmark}{\ding{51}}
\newcommand{\xmark}{\ding{55}}

\newcommand\ti[1]{\textit{#1}}

\newcommand\tf[1]{\textbf{#1}}
\newcommand\ttt[1]{\texttt{#1}}

\newcommand{\smallsec}[1]{\paragraph{#1.}}

\newcommand{\name}{NLProofS }
\newcommand{\namenospace}{NLProofS}

\title{Generating Natural Language Proofs with Verifier-Guided Search}

\author{Kaiyu Yang \and Jia Deng \and Danqi Chen \\
  Department of Computer Science \\
  Princeton University \\
  \texttt{\{kaiyuy,jiadeng,danqic\}@cs.princeton.edu}}

\begin{document}

\maketitle

\begin{abstract}
Reasoning over natural language is a challenging problem in NLP. In this work, we focus on proof generation: Given a hypothesis and a set of supporting facts, the model generates a proof tree indicating how to derive the hypothesis from supporting facts. Compared to generating the entire proof in one shot, stepwise generation can better exploit the compositionality and generalize to longer proofs but has achieved limited success on real-world data. Existing stepwise methods struggle to generate proof steps that are both logically valid and relevant to the hypothesis. Instead, they tend to hallucinate invalid steps given the hypothesis. In this paper, we present a novel stepwise method, \namenospace~(\underline{N}atural \underline{L}anguage \underline{Proof} \underline{S}earch), which learns to generate relevant steps conditioning on the hypothesis. At the core of our approach, we train an independent \emph{verifier} to check the validity of the proof steps to prevent hallucination. Instead of generating steps greedily, we search for proofs maximizing a global proof score judged by the verifier. \name achieves state-of-the-art performance on EntailmentBank and RuleTaker. Specifically, it improves the correctness of predicted proofs from 27.7\% to 33.3\% in the distractor setting of EntailmentBank, demonstrating the effectiveness of \name in generating challenging human-authored proofs.\footnote{The code is available at \url{https://github.com/princeton-nlp/NLProofS}.}
\end{abstract}

\section{Introduction}

A fundamental goal of AI since its inception is automated reasoning~\cite{mccarthy1960programs}: given explicitly provided knowledge as assumptions, we want the system to draw logically valid conclusions. Research in automated reasoning has traditionally focused on structured domains such as formal logic~\cite{robinson2001handbook}. On the other hand, recent work suggests that free-form natural language can be a suitable vehicle for reasoning~\cite{clark2021transformers,dalvi2021explaining}, because natural language represents knowledge without requiring labour-intensive formalization.
However, reasoning in natural language is challenging, as it requires compositional generalization to novel examples~\cite{ruis2020benchmark}---a capability that state-of-the-art large language models struggle with~\cite{rae2021scaling}.

\begin{figure*}[ht]
  \centering
  \includegraphics[width=1.0\linewidth]{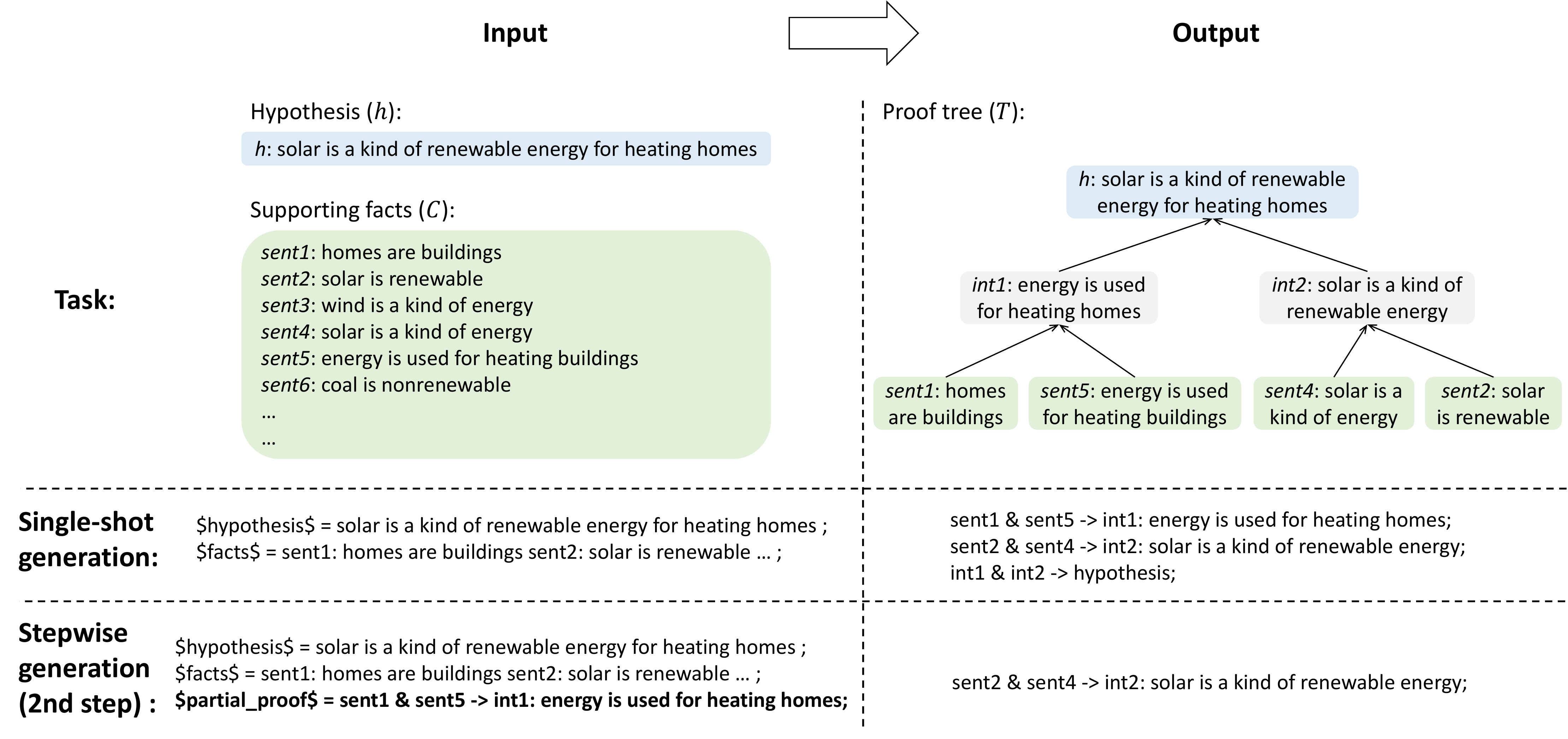}
  \caption{\emph{Top}: In proof generation, given a hypothesis and multiple supporting facts (potentially with redundant facts), the model generates a proof tree, including both the tree structure and the intermediate conclusions (\textit{int1} and \textit{int2}). A common approach encodes the input/output as text sequences and generates the proof in a single-shot (\emph{Middle}) or generates the proof step by step (\emph{Bottom}, showing only one of the three steps) using text-to-text models.
  }
  \label{fig:task}
\end{figure*}

In this work, we focus on proof generation in natural language (Fig.~\ref{fig:task}): given a hypothesis and a set of supporting facts in natural language, the model generates a proof tree indicating how the hypothesis is derived from a subset of the supporting facts. The proof tree may contain intermediate conclusions, which need to be \textit{generated} by the model. Existing methods generate the proof either in a single shot or step by step. Stepwise methods leverage the compositionality of proofs, making it easier for the model to learn and generalize to longer proofs~\cite{tafjord2021proofwriter}.

However, existing stepwise methods suffer from a trade-off between generating \ti{valid} steps and \ti{relevant} steps. Prior works~\citep{sanyal2022fairr,bostrom2022natural} have observed that, given the hypothesis, the model often learns to hallucinate invalid proof steps leading to the hypothesis, instead of performing valid logical inference (see examples in Table~\ref{table:examples}). To mitigate this issue, previous attempts have restricted the model from accessing the hypothesis, forcing it to generate conclusions based solely on known premises. However, without the hypothesis, the model tends to generate many valid but irrelevant steps.
This problem is especially prominent for real-world natural language proofs. Due to the inherent ambiguity of natural language, the search space for each proof step is much larger than that of simple synthetic tasks. That may explain why stepwise methods have demonstrated superior performance on the simple, synthetic RuleTaker dataset~\citep{clark2021transformers} but not on the more realistic, human-authored EntailmentBank dataset~\citep{dalvi2021explaining}, which is the gap we aim to bridge.

We introduce \namenospace, a novel method for stepwise proof generation. It generates proof steps \emph{conditioning on the hypothesis}, enabling the model to learn to generate only relevant steps. To prevent hallucination, it trains an independent \emph{verifier} based on RoBERTa~\citep{liu2019roberta}, which takes a single step (including multiple premises and one conclusion) as input and produces a score indicating its validity. During inference, instead of generating steps greedily, \name searches for proofs that maximize a proof score aggregating the validity scores of all steps.

We evaluate \name on two benchmarks: RuleTaker~\cite{clark2021transformers} and EntailmentBank~\cite{dalvi2021explaining}. RuleTaker consists of simple, synthetic English sentences generated from templates. In contrast, proofs in EntailmentBank are authored by human experts and are more challenging. They are in unconstrained natural language and exhibit considerable fuzziness and ambiguity, as is typical for reasoning in natural language. \name achieves state-of-the-art performance on both datasets. On EntailmentBank, it outperforms previous best results by a large margin. For example, in the distractor task setting, it improves the accuracy of generating complete proofs from 27.7\% to 33.3\% and the accuracy of identifying relevant supporting facts from 46.1\% to 58.8\%, which demonstrates the effectiveness of our method in generating challenging human-authored proofs.

In addition, we conduct extensive ablations to gain insights. First, we show that the verifier plays a crucial role in generating proofs by providing well-calibrated validity scores. Without the verifier, our method performs worse on EntailmentBank and fails completely on RuleTaker. Second, while generating long proofs remains a major challenge, \name leads to large improvements for long proofs. Third, there is still a large room for future improvement, e.g., by generating more accurate and diverse proof steps as candidates for the search algorithm to explore.

\smallsec{Contributions}
In summary, our contributions are two-fold. First, we introduce \namenospace, a stepwise proof generation method that searches for proofs whose validity is scored by a verifier. It substantially advances state-of-the-art performance on the challenging EntailmentBank dataset. Second, through extensive analyses and ablations, we shed light on the performance improvement and reveal the current bottleneck. Our work is a first step exploring the interplay between verifiers and proof search in generating natural language proofs, and we expect further advancements down the road.

\section{Related Work}

%!TEX root = ../paper.tex

\begin{table*}[ht]
  \footnotesize
  \centering
  \makebox[1 \textwidth][c]{
  \resizebox{1 \textwidth}{!}{
  \begin{tabular}{@{}lccccccc@{}}
    \toprule
    Method & \multirow{2}{*}{Stepwise} & Proof & Generate intermediates & \multirow{2}{*}{Verifier} & Non-local & Evaluated on & No external \\
     &  & direction & w/ hypothesis & & search &  human-authored proofs & data \\
    \midrule
    %\citet{gontier2020measuring} & \xmark & N/A & \cmark & \xmark & N/A & \xmark & \cmark \\
    PRover & \xmark & N/A & No intermediates & \xmark & N/A & \xmark & \cmark \\
    %\tentative{PRobr} & \xmark & N/A & No intermediates & \xmark & N/A & \xmark & \cmark \\
    % \tentative{EVR} & \cmark & $\leftarrow$ & \cmark & \xmark & \xmark & \xmark & \cmark \\
    % \tentative{IBR} & \cmark & $\leftarrow$ & No intermediates & \xmark & \xmark & \xmark & \cmark \\
    EntailmentWriter & \xmark & N/A & \cmark & \xmark & N/A & \cmark & \cmark \\
    ProofWriter & \cmark & $\rightarrow$ & \xmark & \xmark & \xmark & \xmark & \cmark \\
    FaiRR & \cmark & $\rightarrow$ & \xmark & \xmark & \xmark & \xmark & \cmark \\
    SCSearch & \cmark & $\rightarrow$ & \xmark & \xmark & \xmark & \xmark & \xmark \\
    MetGen & \cmark & Both & \cmark & \xmark & \xmark & \cmark & \xmark \\
    \citet{dalvi2022towards}\textsuperscript{$\dagger$} & \cmark & $\leftarrow$ & \cmark & \cmark & \xmark & \xmark & \xmark \\
    \midrule
    \name (ours) & \cmark & $\rightarrow$ & \cmark & \cmark & \cmark & \cmark & \cmark \\
    \bottomrule
  \end{tabular}
  }
  }
  \caption{A comparison of \name with existing methods for proof generation: PRover~\citep{saha2020prover}, EntailmentWriter~\citep{dalvi2021explaining}, ProofWriter~\citep{tafjord2021proofwriter}, FaiRR~\citep{sanyal2022fairr},  SCSearch~\citep{bostrom2022natural}, MetGen~\citep{hong2022metgen}, and a concurrent work~\cite{dalvi2022towards} marked with $\dagger$. 
  $\rightarrow$ and $\leftarrow$ denote forward/backward stepwise proof generation.
  }
  \label{table:method_comparison}
\end{table*}

\smallsec{Proof generation in natural language}
Table~\ref{table:method_comparison} summarizes existing methods for generating natural language proofs, including single-shot and stepwise methods. Single-shot methods generate the entire proof tree in one shot, enforcing structural constraints explicitly via linear integer programming~\cite{saha2020prover,sun2021probabilistic} or implicitly via pretrained text-to-text transformers~\cite{gontier2020measuring,dalvi2021explaining} (Fig.~\ref{fig:task}~\emph{Middle}). In contrast, stepwise methods generate the proof as individual proof steps, forward~\citep{tafjord2021proofwriter,sanyal2022fairr,bostrom2022natural}, backward~\citep{liang2021explainable,qu2022interpretable,dalvi2022towards}, or both~\citep{hong2022metgen}. Our method generates proofs stepwise, in the forward direction.

When generating a proof step, prior work has observed that if the hypothesis is available, the model often uses it to hallucinate the intermediate conclusion instead of drawing valid logical inferences (Table~\ref{table:examples}). Therefore, ProofWriter~\citep{tafjord2021proofwriter}, FaiRR~\citep{sanyal2022fairr}, and SCSearch~\citep{bostrom2022natural} explicitly ban the model from accessing the hypothesis when generating intermediate conclusions\footnote{The hypothesis may be used for premise selection.}, forcing it to draw inference from known premises only. However, without the hypothesis, the model may generate many valid proof steps irrelevant to the hypothesis. Unlike other forward stepwise methods, our model has access to the hypothesis but relies on a verifier to check the validity of proof steps and prevent hallucination.

\citet{dalvi2022towards} is a concurrent work that also uses a verifier to score multiple candidate proof steps generated by the model. However, they use the scores to make a greedy local decision, selecting the best step and discarding others, whereas we search for proofs with the maximum aggregated scores. Besides, they train the verifier on additionally annotated negative examples, whereas we train on pseudo-negative examples generated automatically without additional annotation efforts (Sec.~\ref{subsec:verifier}). Other stepwise methods in Table~\ref{table:method_comparison} do not have verifiers, and they make local decisions.

PRover~\cite{saha2020prover}, ProofWriter, and FaiRR have only evaluated on the simple RuleTaker dataset~\cite{clark2021transformers}. And it is nontrivial to extend them to real-world data. For example, FaiRR assumes sentences fall into two categories: rules and facts, which are tailored for RuleTaker. \citet{dalvi2021explaining} introduce EntailmentBank, a challenging benchmark of proofs authored by human experts, which is used to evaluate EntailmentWriter, their method for single-shot proof generation. SCSearch and \citet{dalvi2022towards} also use EntailmentBank but focus on different task settings that do not quantitatively evaluate the generated proofs.

\smallsec{Reasoning in other NLP tasks}
Multi-hop reasoning can also be found in open-domain QA~\cite{yang2018hotpotqa}, fact verification~\cite{jiang2020hover}, and reading comprehension~\cite{min2019multi,sinha2019clutrr,jiang2019explore}. Compared to proof generation, reasoning chains in these tasks are much simpler, often consisting of only 2--3 supporting facts. Also, they are more coarse-grained, involving large chunks of texts such as passages instead of simple, short sentences.

\citet{bostrom2021flexible} generate conclusions from premises. Their method can potentially be a component in proof generation but does not consider whether the generated conclusions are relevant.
In math word problems, \citet{cobbe2021training} demonstrate the benefits of using a verifier to re-rank the model's predicted solutions. However, these solutions are unconstrained texts, whereas proofs in our task are structured trees/graphs. Further, we use the verifier during proof generation rather than merely to rank the solutions post hoc. Our verifier is also related to natural language inference~\citep{bowman2015large}, especially the multi-premises setting in \citet{lai2017natural}. 
Recently, large language models have shown the ability to solve multi-step reasoning through chain-of-thought prompting~\citep{wei2022chain,kojima2022large} on arithmetic, symbolic and commonsense reasoning tasks.

\smallsec{Symbolic reasoning}
Classical AI has invested significant efforts in reasoning in symbolic domains, e.g., automated theorem proving (ATP)~\cite{kovacs2013first,yang2019learning,polu2020generative}. Researchers have attempted to apply ATP to natural language through semantic parsing~\cite{mineshima2015higher,saparov2021generative}. However, it is challenging (if not impossible) for semantic parsers to cover the full complexity of natural language. Therefore, researchers have developed reasoning approaches bypassing semantic parsing~\cite{angeli2016combining,kalyanpur2020braid,yang2021learning}.

One promising example is neurosymbolic reasoning. It uses neural networks to handle the complexity of natural language but incorporates inductive biases inspired by symbolic reasoning~\cite{weber2019nlprolog,smolensky1990tensor,kathryn2018tensorlog,lee2015reasoning}. 
Our method also falls into this broad category. It uses large language models to generate individual reasoning steps but chains the steps together into a coherent, tree-structured proof using symbolic search algorithms.
\section{Generating Natural Language Proofs}
\label{sec:task}

\smallsec{Task definition}
Now we define the proof generation task. As in Fig.~\ref{fig:task}~(\emph{Top}), the input consists of a hypothesis $h$ and a set of supporting facts $C = \{\mathrm{sent}_1, \mathrm{sent}_2, \dots, \mathrm{sent}_n\}$. Both $h$ and $\mathrm{sent}_i$ are natural language sentences. $h$ can be derived from a subset of $C$ through reasoning of one or multiple steps.

The output is a proof tree $T$ specifying how $h$ is derived from $C$. The tree has $h$ as its root and $\mathrm{sent}_i$ as leaf nodes. The intermediate nodes are intermediate conclusions \ti{generated} by the model. Each non-leaf node $u$ corresponds to a reasoning step with $u$ as the conclusion and its children as premises. To successfully perform the task, the model must select relevant sentences from $C$, use them as leaf nodes to compose a valid proof tree leading to $h$, and fill in all the intermediate conclusions.

\smallsec{Singles-shot vs. stepwise generation}
A simple and effective method for proof generation, popularized by ProofWriter~\cite{tafjord2021proofwriter}, is to finetune a pretrained T5 model~\cite{raffel2020exploring} to map the input ($h$ and $C$) to the output ($T$), either in a single shot or stepwise. To that end, the input/output must be encoded as text sequences, e.g., encoding the input by concatenating $h$ and $C$ as illustrated in Fig.~\ref{fig:task}.\footnote{Our input encoding scheme has minor differences from EntailmentWriter~\citep{dalvi2021explaining} (details in Appendix~\ref{sec:entailmentbank_format}).
}

The output proof tree can be encoded by post-order traversal. As in Fig.~\ref{fig:task}, nodes are labeled with identifiers: \ttt{sent*} for leaf nodes, \ttt{int*} for intermediate nodes, and \ttt{hypothesis} for the root. The output sequence is produced by traversing the tree in post-order, generating one proof step at each non-leaf node, using \texttt{\&} for ``and'' and \texttt{->} for ``entails''. The tree may correspond to multiple valid sequences due to different ordering between proof steps and between premises within a step. Nevertheless, the evaluation metric can be calculated from the reconstructed trees instead of the raw text sequences.

In single-shot generation, the model generates the output sequence of the entire proof (Fig.~\ref{fig:task} \emph{Middle}), whereas in stepwise generation, each time the model takes the current partial proof as input (besides $h$ and $C$) and generates only the next step (Fig.~\ref{fig:task} \emph{Bottom}).

\section{Our Method: {\name}}

Now we present \namenospace, our method for generating natural language proofs. It has three main components: (1) a stepwise prover for generating candidate proof steps; (2) a verifier for scoring the validity of proofs; (3) an algorithm for searching for proofs that have high aggregated proof scores.

\subsection{Stepwise Prover}
\label{subsec:stepwise}

Like prior work~\cite{tafjord2021proofwriter}, we implement the stepwise prover by finetuning a pretrained T5 model. The training data is extracted from the steps in ground truth proofs. Let $T$ be a proof tree and $u \in T$ be a non-leaf node corresponding to a step we want to extract. Take node \ttt{int1} in Fig.~\ref{fig:task} as an example of $u$. Non-leaf nodes in $T$ can be categorized into (1) $u$'s descendants, e.g., none in Fig.~\ref{fig:task}; (2) $u$ itself and its ancestors, e.g., \ttt{int1} and \ttt{h} in Fig.~\ref{fig:task}; (3) neither, e.g., \ttt{int2} in Fig.~\ref{fig:task}. The partial proof must include all of (1) but none of (2). It may or may not include nodes in (3). Therefore, for this particular example, the partial proof cannot include \ttt{int1} or \ttt{h} but has a free choice about whether to include \ttt{int2}. When preprocessing the training data, we make these choices randomly as a form of data augmentation.

During inference, the prover may generate syntactically ill-formed proof steps. For the example in Fig.~\ref{fig:task}~(\emph{Bottom}), ``\texttt{int1 \& int2 -> hypothesis};'' is ill-formed, since the premise \texttt{int2} is not available in the current partial proof. We mitigate the issue by generating multiple proof steps from the model via beam search and using heuristics to filter out ill-formed ones, e.g., those with syntactical errors or unavailable premises.

\begin{figure*}[ht]
  \centering
  \includegraphics[width=1.0\linewidth]{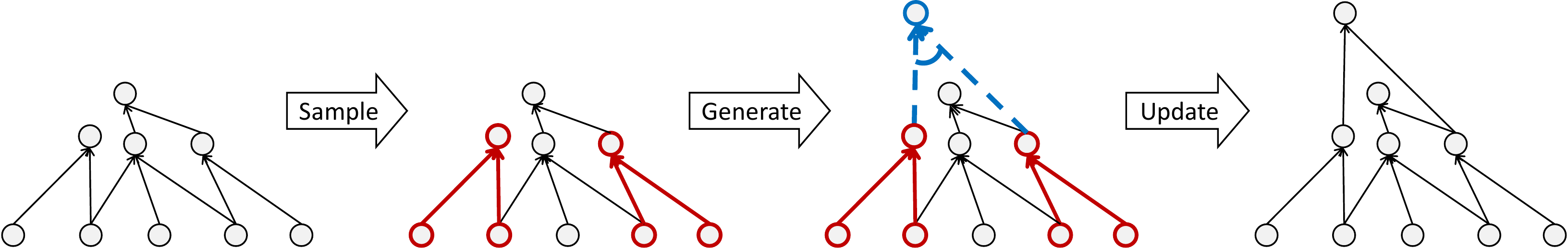}
  \caption{An iteration in the proof search. Nodes for proof step ($S$) are omitted for simplicity of illustration. (1) Sample a partial proof (red) from the proof graph. (2) Use the stepwise prover to generate potential steps (blue, only showing one but generating multiple) and score them using the verifier. (3) Execute the steps to update the graph.}
  \label{fig:graph}
\end{figure*}

\subsection{Verifier}
\label{subsec:verifier}

\smallsec{Scoring a proof step}
We introduce an independent verifier, which is trained to check the validity of proof steps and prevent the prover from hallucinating invalid steps based on the hypothesis. A proof step has multiple premises and one conclusion. The verifier takes them as input and produces a continuous validity score in $[0, 1]$.

We implement the verifier by finetuning a pretrained RoBERTa model~\cite{liu2019roberta} to classify proof steps as valid or invalid. The input premises are shuffled randomly and concatenated with the conclusion. For training data, positive examples (valid steps) can be extracted from ground-truth proofs; however, there are no negative examples readily available. Instead of annotating additional negative examples as in \citet{dalvi2022towards}, we generate pseudo-negative examples automatically. Please refer to Appendix~\ref{sec:negatives} for details. 

\smallsec{Aggregating scores for the entire proof}
Step scores are aggregated to produce the score of the entire proof tree. We associate scores with all nodes in the tree recursively. All leaves have a score of 1.0, as they are explicitly provided assumptions that always hold. Each non-leaf node $u$ corresponds to a proof step $s$ from its children $v_1, v_2, \dots, v_l$ and has a score defined as
\begin{equation}
\label{eqn:proof_score}
\texttt{scr}_n(u) = \mathrm{min}\big(\texttt{scr}_s(s), \texttt{scr}_n(v_1), \dots, \texttt{scr}_n(v_l)\big),
\end{equation}
where $\texttt{scr}_s(s)$ is the step score, e.g., produced by a verifier. Intuitively, $\texttt{scr}_n(u)$ reflects our confidence in $u$, and it is monotonically non-increasing w.r.t. the step score and the scores of its children. Eqn.~\ref{eqn:proof_score} is just one simple way of defining $\texttt{scr}_n(u)$, and we leave a more thorough exploration of scoring options for future work. Finally, the proof score is $\texttt{scr}_n(h)$: the root's score.

\subsection{Proof Search}
\label{subsec:proof_search}

Now we combine the prover and the verifier in our proof search algorithm, which looks for proofs with optimal proof scores. Our method is inspired by automated reasoning in formal logic~\cite{russell2002artificial}, where proofs are found by searching in a large space efficiently. Instead of greedy stepwise proof generation, we search for proof trees in a large proof graph (Fig.~\ref{fig:graph}), allowing the model to explore different paths, recover from errors, and ultimately find better proofs.

\begin{definition}[Proof graph]
\label{dfn:proof_graph}
A proof graph is a directed acyclic graph with the following properties:
\begin{itemize}
    \item \emph{Nodes}: It has four types of nodes $(C, I, S, h)$, where $C$ corresponds to supporting facts, $I$ corresponds to intermediate conclusions, $S$ consists of proof steps, and $h$ is the hypothesis. Nodes in $\{h\} \bigcup I \bigcup C$ are associated with unique sentences.
    \item \emph{Edges}: For any proof step node $s \in S$, it has one or more inbound edges, all of which originate from $I \bigcup C$. It has exactly one outbound edge, which points to a node in $I \bigcup \{h\}$. Besides these edges, the graph contains no additional edges. Any node $u \in I \bigcup \{h\}$ has at most one inbound edge.
    \item \emph{Scores}: All nodes are associated with scores in $[0, 1]$. For any sentence $\mathrm{sent}_i \in C$, $\texttt{scr}_n(\mathrm{sent}_i) = 1$. For any node $u \in I \bigcup \{h\}$, $\texttt{scr}_n(u) = 0$ if it has no inbound edge. Otherwise, it must have exactly one inbound edge from $s \in S$, and $s$ has inbound edges from $\{v_1, \dots, v_l\} \subseteq I \bigcup C$. $\texttt{scr}_n(u)$ is defined by Eqn.~\ref{eqn:proof_score}. Scores of proof step nodes in $S$ are provided externally by the algorithm that operates on the proof graph.
\end{itemize}
\end{definition}
Proof trees correspond to paths in proof graphs (treating $S$ as ``and'' nodes in and-or graphs). Therefore, our task is to search for a path from $C$ to $h$ that maximizes $\texttt{scr}_n(h)$. The search algorithm is outlined in Fig.~\ref{fig:graph} and Algorithm~\ref{alg:search}. Proof search takes place only in inference. In training, we train a stepwise prover $\mathcal{P}$ and a verifier $\mathcal{V}$. In inference, we use them to iteratively expand the proof graph and update the node scores until the graph can no longer be updated. At that point, we extract the best proof of $h$ found so far.

\setlength{\intextsep}{2pt}
\begin{algorithm}[ht]
\DontPrintSemicolon
\SetKwInOut{Input}{Input}
\SetKwInOut{Output}{Output}
\SetKwFunction{GenerateGreedyProof}{generate\_greedy}
\SetKwFunction{SampleNewPartialProof}{sample\_new}
\SetKwFunction{InitializeProofGraph}{initialize\_graph}
\SetKwFunction{ScoreProofSteps}{verify}
\SetKwFunction{GenerateProofSteps}{generate}
\SetKwFunction{UpdateGraph}{update}
\SetKwFunction{ExtractProof}{extract\_proof}
\Input{Hypothesis $h$, supporting facts $C$, stepwise prover $\mathcal{P}$, verifier $\mathcal{V}$}
\Output{Proof tree $T$}
$\mathcal{G} \leftarrow \GenerateGreedyProof(\mathcal{P}, h, C)$ \;
$\mathcal{PG} \leftarrow \InitializeProofGraph(\mathcal{G})$ \;
$\text{explored} \leftarrow \varnothing$ \;
\While{true}{
    $\text{partial\_proof} \leftarrow \SampleNewPartialProof(\mathcal{PG}, \text{explored})$ \;
    $\text{explored} \leftarrow \text{explored} \cup \{\text{partial\_proof}\}$ \;
    $\text{steps}, \text{p\_scrs} \leftarrow \GenerateProofSteps(\mathcal{P}, \text{partial\_proof})$ \;
    $\text{v\_scrs} \leftarrow \ScoreProofSteps(\mathcal{V}, \text{steps})$ \;
    $\text{scrs} \leftarrow (\text{p\_scrs} + \text{v\_scrs}) / 2$ \;
    $\mathcal{PG}' \leftarrow \UpdateGraph(\mathcal{PG}, \text{steps}, \text{scrs})$ \;
    \If{$\mathcal{PG}' = \mathcal{PG} $}{
        break \;
    }
    $\mathcal{PG} \leftarrow \mathcal{PG}'$ \;
}
\KwRet{$\ExtractProof(\mathcal{PG})$}
 \caption{Proof search.}
 \label{alg:search}
\end{algorithm}

\smallsec{Initialization (line 1--3 in Algorithm~\ref{alg:search})}
We initialize the proof graph using the greedy proof generated by $\mathcal{P}$. We could also start from scratch, i.e., $I = S = \varnothing$ and $\texttt{scr}(h) = 0$, but the initialization accelerates proof search by providing a non-zero initial score for $h$, which can be used to prune unpromising paths during search.

\smallsec{Iteration (line 5--13 in Algorithm~\ref{alg:search})}
We use $\mathcal{P}$ to generate proof steps for updating the graph. $\mathcal{P}$ is trained on partial proof trees rather than graphs. So in each iteration, we first sample a new partial proof tree from the graph as the candidate for expansion (details in Appendix~\ref{sec:sampling}). Then, we use $\mathcal{P}$ to generate multiple proof steps $s_1, s_2, \dots, s_k$ through beam search followed by filtering as discussed in Sec.~\ref{subsec:stepwise}. We calculate step scores $\texttt{scr}_s(s_1), \texttt{scr}_s(s_2), \dots, \texttt{scr}_s(s_k)$ by averaging verifier scores v\_scrs from $\mathcal{V}$ (Sec.~\ref{subsec:verifier}) and prover scores p\_scrs from $\mathcal{P}$, which are the likelihood scores in beam search.

Then we try to update the proof graph by executing these steps. Assume a step $s_i$ has premises $v_1, \dots, v_l$ and a conclusion $u$. First, we use Eqn.~\ref{eqn:proof_score} to calculate a tentative score $\widehat{\texttt{scr}}_n(u)$. If $u$ is an existing node in the graph with $\texttt{scr}_n(u) \geq \widehat{\texttt{scr}}_n(u)$, the step becomes a no-op, and we do not perform any update. Otherwise, there are two cases: (1) If $u$ is not in the graph (Fig.~\ref{fig:graph}), we just create a new node for it with $\texttt{scr}_n(u) = \widehat{\texttt{scr}}_n(u)$; (2) If $u$ is in the graph and $\texttt{scr}_n(u) < \widehat{\texttt{scr}}_n(u)$, we update $u$ by replacing the existing proof step leading to it with the new step $s_i$ with $\texttt{scr}_n(u) = \widehat{\texttt{scr}}_n(u)$. According to Eqn.~\ref{eqn:proof_score}, the score change may affect $u$'s successors, so we propagate the change to all of them.

\smallsec{Proof extraction (line 14 in Algorithm~\ref{alg:search})}
When all proof steps in an iteration are no-op, we stop and extract the best proof of $h$ found so far, which simply consists of all predecessors of $h$. The result is guaranteed to be a tree, as we prove in Appendix~\ref{sec:tree}.

\section{Main Results}
\label{sec:main_results}

\subsection{Experimental Setup}

We evaluate \name on proof generation using two benchmarks:  a real-world benchmark EntailmentBank~\cite{dalvi2021explaining} and a synthetic benchmark RuleTaker~\cite{clark2021transformers}.
Training and inference details are in Appendix~\ref{sec:details}. \citet{bostrom2022natural} also evaluated on EntailmentBank but deviated from the original setting, instead formulating the task as distinguishing verifiable hypotheses from unverifiable ones. In order to have a fair comparison with their work, we also evaluate under their setting in Appendix~\ref{sec:bostrom}.

%!TEX root = ../paper.tex

\begin{table*}[ht]
  \centering
  \vspace{-5mm}
  \makebox[1 \textwidth][c]{
  \resizebox{1 \textwidth}{!}{
  \begin{tabular}{@{}lllllllll@{}}
    \toprule
     Task & Method & \multicolumn{2}{c}{Leaves} & \multicolumn{2}{c}{Steps} & \multicolumn{2}{c}{Intermediates} & Overall \\
     \cmidrule(r){3-4} \cmidrule(r){5-6} \cmidrule(r){7-8} & & F1 & AllCorrect & F1 & AllCorrect & F1 & AllCorrect & AllCorrect  \\
    \midrule
    Task 1 & EntailmentWriter & 98.7 & 86.2 & 50.5 & 37.7 & 67.6 & 36.2 & 33.5 \\
    (no-distractor) & EntailmentWriter (T5-11B) & \underline{99.0} & 89.4 & 51.5 & 38.2 & \underline{71.2} & 38.5 & 35.3 \\
    & MetGen$^\dagger$ & \textbf{100.0} & \textbf{100.0} & \textbf{57.7} & \underline{41.9} & 70.8 & \underline{39.2} & \underline{36.5} \\
    & \name (ours) & 97.8 $\pm$ 0.2 & \underline{90.1 $\pm$ 1.2} & \underline{55.6 $\pm$ 0.6} & \textbf{42.3 $\pm$ 0.4} & \textbf{72.4 $\pm$ 0.5} & \textbf{40.6 $\pm$ 0.7} & \textbf{38.9 $\pm$ 0.7} \\
    \midrule
    Task 2 & EntailmentWriter & 84.3 & 35.6 & 35.5 & 22.9 & 61.8 & 28.5 & 20.9 \\
    (distractor) & EntailmentWriter (T5-11B) & \underline{89.1} & \underline{48.8} & \underline{41.4} & 27.7 & \underline{66.2} & 31.5 & 25.6 \\
    & MetGen$^\dagger$ & 82.7 & 46.1 & 41.3 & \underline{29.6} & 61.4 & \underline{32.4} & \underline{27.7} \\
    & \name (ours) & \textbf{90.3 $\pm$ 0.4} & \textbf{58.8 $\pm$ 1.8} & \textbf{47.2 $\pm$ 1.7} & \textbf{34.4 $\pm$ 1.7} & \textbf{70.2 $\pm$ 0.5} & \textbf{37.8 $\pm$ 1.6} & \textbf{33.3 $\pm$ 1.5} \\
    \midrule
    Task 3 & EntailmentWriter & 35.7 & 2.9 & 6.1 & 2.4 & 33.4 & 7.7 & 2.4 \\
    (full-corpus) & EntailmentWriter (T5-11B) & \underline{39.9} & 3.8 & 7.4 & 2.9 & 35.9 & 7.1 & 2.9 \\
    & MetGen$^\dagger$ & 34.8 & \textbf{8.7} & \underline{9.8} & \textbf{8.6} & \underline{36.6} & \textbf{20.4} & \textbf{8.6} \\
    & \name (ours) & \textbf{43.2 $\pm$ 0.6} & \underline{8.2 $\pm$ 0.7} & \textbf{11.2 $\pm$ 0.6} & \underline{6.9 $\pm$ 0.7} & \textbf{42.9 $\pm$ 1.0} & \underline{17.3 $\pm$ 0.5} & \underline{6.9 $\pm$ 0.7} \\
    \bottomrule
  \end{tabular}
  }
  }

  \caption{Test results of proof generation on EntailmentBank~\citep{dalvi2021explaining}. 
  $^\dagger$: MetGen~\citep{hong2022metgen} is trained on \ti{additional data} collected from Wikipedia, whereas other methods are trained only on EntailmentBank. Here we report the results of the MetGen-prefixed model, as the other MetGen-separated model performs slightly better but is 5 times larger. All methods are based on T5-large~\cite{raffel2020exploring} unless otherwise noted. For our method, we report the average performance and the standard deviation for 5 independent runs. Bold and underlined texts highlight the best method and the runner-up.
  }
  \label{table:entailmentbank_main_test}
\end{table*}

%!TEX root = ../paper.tex

\begin{table*}[ht]
  \centering
  \makebox[\textwidth][c]{
  \resizebox{0.9 \textwidth}{!}{
  \begin{tabular}{p{3.5cm}p{4.5cm}p{6cm}}
    \toprule
     Hypothesis & Premises & Conclusions generated by models \\
    \midrule
    \multirow{2}{=}{The next new moon will occur on June 30.} & \multirow{2}{=}{1. A new moon is a kind of phase of the moon. 2. A moon phase occurs 28 days after the last time it occurs.} & \tf{EntailmentWriter}: The next new moon will occur 28 days after June 2. \\
    \cmidrule(r){3-3} & & \tf{\name (ours)}: The next new moon will occur 28 days after the last new moon. \\
    \midrule
     \multirow{2}{=}{Planting trees prevents soil from washing away.} & \multirow{2}{=}{1. Planting trees increases the amount of trees in an environment. 2. Tree roots decrease / reduce soil erosion.} & \tf{EntailmentWriter}: Plants trees increases the amount of trees in an environment. \\
    \cmidrule(r){3-3} & & \tf{\name (ours)}: Planting trees decreases soil erosion. \\
    \bottomrule
  \end{tabular}
  }
  }
  \caption{Examples of invalid proof steps generated by EntailmentWriter~\cite{dalvi2021explaining} but not our method. In the first example, ``June'' in the conclusion is hallucinated rather than derived from the premises. In the second example, EntailmentWriter simply copies one of the premises without performing any meaningful reasoning. 
  }
  \label{table:examples}
\end{table*}

\begin{table*}[ht]
  \footnotesize
  \centering 
  \vspace{-2mm}
  \begin{tabular}{@{}lllllllllllll@{}}
    \toprule
     Method & \multicolumn{6}{c}{Answer accuracy} & \multicolumn{6}{c}{Proof accuracy} \\
     \cmidrule(r){2-7} \cmidrule(r){8-13} & N/A & 0 & 1 & 2 & 3 & All & N/A & 0 & 1 & 2 & 3 & All \\
    \midrule
    FaiRR\textsuperscript{$\dagger$} & \underline{99.6} & \textbf{100.0} & 99.7 & 98.9 & \underline{96.6} & 99.2 & \underline{99.6} & \textbf{100.0} & 99.5 & 97.2 & \underline{95.3} & 98.8 \\
    ProofWriter\textsuperscript{$\dagger$} & \textbf{99.7} & \textbf{100.0} & \underline{99.9} & \textbf{99.7} & \textbf{99.7} & \textbf{99.8} & \textbf{99.7} & \textbf{100.0} & \underline{99.9} & \textbf{99.4} & \textbf{99.1} & \textbf{99.7} \\
    \name (ours) & 99.5 & \textbf{100.0} & \textbf{100.0} & \underline{99.4} & 96.4 & \underline{99.3} & 99.5 & \textbf{100.0} & \textbf{100.0} & \textbf{99.4} & 95.1 & \underline{99.2} \\
    \bottomrule
  \end{tabular}
  \caption{Test results on RuleTaker (OWA)~\cite{tafjord2021proofwriter}. Models are trained and tested on the D0--D3 subset. Methods with $\dagger$ are reported by \citet{sanyal2022fairr}. All methods are based on T5-large. The `All' columns are accuracies on the entire testing set. `0', `1', `2', and `3' are accuracies broken down by the length of testing proofs. `N/A' includes testing examples without ground truth proofs since they can be neither proved nor disproved.}
  \label{table:ruletaker_main}
\end{table*}

\smallsec{EntailmentBank}
EntailmentBank consists of 1,840 proof trees constructed by expert annotators (1,313 for training, 187 for validation, and 340 for testing). It comes with three variants of the proof generation task (Sec.~\ref{sec:task}) with varying numbers of distractors in supporting facts $C$. \emph{Task 1} does not have any distractor, i.e., $C$ consists of exactly the leaf nodes of the ground truth proof tree. In \emph{Task 2}, $C$ always has 25 sentences, including ground truth supporting facts as well as distractors. In \emph{Task 3}, $C$ is a large corpus of 12K sentences derived from WorldTree V2~\cite{xie2020worldtree}, requiring the model to retrieve relevant supporting facts from the corpus. We evaluate on all three tasks. Our method is directly applicable to Task 1 and Task 2. For Task 3, \citet{dalvi2021explaining} retrieve 25 supporting facts for each hypothesis. We use the same retrieved supporting facts and focus solely on proof generation.\footnote{Appendix~\ref{sec:additional_experiments} includes additional discussion on prior work focusing on improving the retriever~\cite{ribeiro2022entailment}.} And following their practice, we train the model on Task 2 and evaluate its zero-shot performance on Task 3. 

A generated proof tree $\hat{T}$ is compared against the ground truth $T$ using official metrics developed by EntailmentBank. In summary, the leaves, proof steps, and intermediate conclusions in $\hat{T}$ are compared against those in $T$ to produce two metrics: the F1 score, and the AllCorrect score which evaluates exact matches.\footnote{Intermediates are compared using BLEURT~\cite{sellam2020bleurt}: a learning-based sentence similarity measure.} In addition, the Overall-AllCorrect metric measures whether $\hat{T}$ is identical to $T$. As a caveat, these metrics do not account for the existence of multiple valid proof trees. Metrics for evaluating leaves are less impacted by this issue, as multiple valid trees often have the same set of leaves. Please refer to Appendix~\ref{sec:entailmentbank_metrics} and EntailmentBank for additional details. We report results produced by their official evaluation code.\footnote{\url{https://github.com/allenai/entailment_bank}}

\smallsec{RuleTaker}
To demonstrate the broad applicability of \name to different reasoning datasets, we also evaluate on RuleTaker. In RuleTaker, $h$ can be either proved, disproved, or neither. The model has to do two things: (1) predict the answer as one of those three categories and (2) generate a proof when $h$ can be proved or disproved. Unlike EntailmentBank, examples in RuleTaker are made of simple, synthetic English sentences generated by templates. We use the OWA (open-world assumption) version of the dataset introduced by \citet{tafjord2021proofwriter}. Following the setup in \citet{sanyal2022fairr}, we train and test on the D0--D3 subset, which consists of proofs of depth $\leq 3$. It has 129K examples---90K for training, 13K for validation, and 26K for testing.

The predicted answer is evaluated using accuracy, whereas proofs are evaluated using Overall-AllCorrect but ignoring the intermediate conclusions.\footnote{The metric was introduced by PRover~\citep{saha2020prover}.}

\subsection{Proof Generation on EntailmentBank}

Table~\ref{table:entailmentbank_main_test} shows test results on EntailmentBank. We compare with EntailmentWriter~\cite{dalvi2021explaining} and MetGen~\cite{hong2022metgen}: two prior state-of-the-art methods that also finetune a T5 model to generate proofs. EntailmentWriter generates the entire proof in a single shot, whereas MetGen generates the proof stepwise. EntailmentWriter has two versions, one with T5-large (737 million parameters) and the other with T5-11B (11 billion parameters). All other methods, including ours, use only T5-large due to computational constraints.  

\name significantly outperforms EntailmentWriter across the board. Take Task 2 as an example. First, it generates more correct proofs overall, improving the Overall-AllCorrect metric from 20.9\% to 33.3\%. Second, it identifies relevant supporting facts more effectively, improving the Leaves-AllCorrect from 35.6\% to 58.8\%. Third, it generates more accurate proof steps and intermediate conclusions, as demonstrated by the Steps and Intermediates metrics. Moreover, our method with T5-large even outperforms EntailmentWriter with T5-11B by a large margin. 

Compared to MetGen, we perform competitively on Task 1 and Task 3 but much better on Task 2, improving the Overall-AllCorrect metric from 27.7\% to 33.3\%. Note that our model is trained only on EntailmentBank, whereas MetGen requires much more data annotation efforts (Sec. 4.1.2 in \citet{hong2022metgen}). First, the MetGen authors manually design templates of different reasoning types and use them to collect additional training data from Wikipedia. Second, they manually annotate the reasoning types of 400 training proof steps in EntailmentBank. MetGen needs these annotations since the model takes the reasoning type as input.

We also examine whether the proof can be generated in a single shot by very large language models such as  GPT-3~\citep{brown2020language} or Codex~\citep{chen2021evaluating}, through prompting with in-context examples. Results in Appendix~\ref{sec:additional_experiments} show that in-context prompting performs substantially worse than our method.

Table~\ref{table:examples} shows two examples of invalid steps generated by EntailmentWriter but avoided by \name, likely due to its verifier. In the first example, ``June'' in EntailmentWriter's conclusion is hallucinated based on the hypothesis, as the word does not appear in the premises. The second example is a typical undesirable behavior also observed by \citet{bostrom2022natural}. When the model has difficulties in generating a conclusion, it falls back into copying one of the premises. Our method generates reasonable conclusions in these two examples.

\begin{table*}[ht]
  \centering 
  \makebox[1 \textwidth][c]{
  \resizebox{1 \textwidth}{!}{
  \begin{tabular}{@{}lllllllll@{}}
    \toprule
     Method & \multicolumn{2}{c}{Leaves} & \multicolumn{2}{c}{Steps} & \multicolumn{2}{c}{Intermediates} & Overall & Time \\
     \cmidrule(r){2-3} \cmidrule(r){4-5} \cmidrule(r){6-7} & F1 & AllCorrect & F1 & AllCorrect & F1 & AllCorrect & AllCorrect & \\
    \midrule
    \name (full model) & \textbf{90.3 $\pm$ 0.4} & \textbf{58.8 $\pm$ 1.8} & \textbf{47.2 $\pm$ 1.7} & \textbf{34.4 $\pm$ 1.7} & \textbf{70.2 $\pm$ 0.5} & \textbf{37.8 $\pm$ 1.6} & \textbf{33.3 $\pm$ 1.5} & 4.4 \\
    \quad w/o search & 89.7 $\pm$ 0.6 & 56.5 $\pm$ 1.7 & \underline{45.9 $\pm$ 1.3} & 33.7 $\pm$ 1.4 & 67.4 $\pm$ 2.3 & \underline{36.4 $\pm$ 1.5} & 31.8 $\pm$ 1.4 & \underline{2.2} \\
    \quad w/o search w/o stepwise & 86.9 $\pm$ 0.6 & 45.6 $\pm$ 1.5 & 42.6 $\pm$ 1.6 & 29.7 $\pm$ 1.3 & 64.6 $\pm$ 1.4 & 32.2 $\pm$ 2.1 & 27.1 $\pm$ 1.5 & \textbf{1.9} \\ 
    \quad  w/o prover score & \textbf{90.3 $\pm$ 0.7} & \underline{57.8 $\pm$ 1.9} & 43.9 $\pm$ 0.9 & 30.4 $\pm$ 0.5 & \underline{68.9 $\pm$ 0.5} & 35.3 $\pm$ 1.2 & 29.7 $\pm$ 0.8 & 4.6 \\
    \quad  w/o verifier score & 89.7 $\pm$ 0.6 & 55.8 $\pm$ 2.2 & 45.8 $\pm$ 1.4 & \underline{33.8 $\pm$ 1.5} & 68.5 $\pm$ 0.4 & 36.1 $\pm$ 1.4 & \underline{31.9 $\pm$ 1.3} & 3.3 \\
    \bottomrule
  \end{tabular}
  }
  }
  \caption{Ablation results on Task 2 of the test set of EntailmentBank. The last column shows the average inference time (in seconds) per test example when running with a batch size of 1 and a beam width of 10.}
  \label{table:entailmentbank_ablation_test}
\end{table*}
\begin{table*}[ht]
  \footnotesize
  \centering 
  \begin{tabular}{@{}lllllllllllll@{}}
    \toprule
     Method & \multicolumn{6}{c}{Answer accuracy} & \multicolumn{6}{c}{Proof accuracy} \\
     \cmidrule(r){2-7} \cmidrule(r){8-13} & N/A & 0 & 1 & 2 & 3 & All & N/A & 0 & 1 & 2 & 3 & All \\
    \midrule
     \name (full model) & \underline{99.5} & \textbf{100.0} & \textbf{100.0} & \textbf{99.4} & \textbf{96.4} & \textbf{99.3} & \underline{99.5} & \textbf{100.0} & \textbf{100.0} & \textbf{99.4} & \textbf{95.1} & \textbf{99.2} \\
    \quad w/o search & 89.2 & 55.9 & 41.7 & 33.7 & 36.8 & 64.3 & 89.2 & 99.6 & 81.2 & 62.4 & 65.0 & 84.2 \\
    \quad w/o search w/o stepwise & 77.5 & 49.9 & 27.1 & 30.4 & 30.4 & 54.8 & 77.5 & 97.0 & 49.2 & 59.8 & 56.9 & 72.7 \\
    \quad w/o prover score & \textbf{99.6} & \textbf{100.0} & \underline{99.9} & \underline{98.9} & \underline{96.2} & \textbf{99.3} & \textbf{99.6} & \underline{99.8} & \underline{99.8} & \underline{98.8} & \underline{94.3} & \underline{99.0} \\
    \quad w/o verifier score & 86.5 & 55.6 & 47.2 & 47.2 & 52.2 & 67.0 & 86.5 & 98.9 & 93.3 & 95.5 & 92.5 & 91.4 \\
    \bottomrule
  \end{tabular}
  \caption{Ablation results on the D0--D3 test set of RuleTaker (OWA)~\cite{tafjord2021proofwriter}.}
  \label{table:ruletaker_ablation}
\end{table*}

\subsection{Generating Answers and Proofs on RuleTaker}

\label{subsec:ruletaker_adaptation}
Hypotheses in RuleTaker can be provable, disprovable, or neither. To benchmark on RuleTaker, we use a similar scheme to \citet{bostrom2022natural} to adapt any proof generation system capable of producing proof scores. In training, we (1) discard hypotheses that are neither provable nor disprovable and (2) convert disprovable hypotheses into provable ones by negating them. We negate sentences by adding an ``I don't think'' prefix. 

In testing, given a hypothesis $h$, we try to generate proofs and the associated scores for both $h$ and its negation $\neg h$. Then we train a linear classifier on top of the two scores to predict the answer. Depending on the predicted answer, we take the generated proof to be the proof of $h$, $\neg h$, or neither.

Results are in Table~\ref{table:ruletaker_main}. ProofWriter is the iterative model in \citet{tafjord2021proofwriter}. It generates proofs stepwise based on T5. Our method performs competitively with ProofWriter and FaiRR~\citep{sanyal2022fairr}.

\begin{table*}[ht]
   \centering
  \makebox[1 \textwidth][c]{
  \resizebox{1 \textwidth}{!}{
  \begin{tabular}{@{}llllllll@{}}
    \toprule
     Method & \multicolumn{2}{c}{Leaves} & \multicolumn{2}{c}{Steps} & \multicolumn{2}{c}{Intermediates} & Overall \\
     \cmidrule(r){2-3} \cmidrule(r){4-5} \cmidrule(r){6-7} & F1 & AllCorrect & F1 & AllCorrect & F1 & AllCorrect & AllCorrect  \\
    \midrule
    \name (no oracle) & 89.4 $\pm$ 0.8 & 56.0 $\pm$ 0.7 & 50.4 $\pm$ 1.9 & 38.4 $\pm$ 1.3 & 71.9 $\pm$ 1.4 & 41.3 $\pm$ 1.4 & 37.1 $\pm$ 1.5 \\
     \quad oracle verifier & 90.0 $\pm$ 0.5 & 56.9 $\pm$ 1.6 & 51.3 $\pm$ 2.2 & 39.1 $\pm$ 2.1 & 72.9 $\pm$ 1.5 & 42.0 $\pm$ 2.1 & 37.9 $\pm$ 2.2 \\
    \quad oracle prover & \underline{94.6 $\pm$ 0.3} & \underline{76.2 $\pm$ 1.9} & \underline{75.8 $\pm$ 0.9} & \underline{67.0 $\pm$ 1.8} & \underline{85.3 $\pm$ 0.5} & \underline{67.5 $\pm$ 1.6} & \underline{66.3 $\pm$ 1.9} \\
    \quad oracle prover + verifier & \textbf{95.7 $\pm$ 0.7} & \textbf{82.7 $\pm$ 2.7} & \textbf{83.0 $\pm$ 3.1} & \textbf{75.4 $\pm$ 3.9} & \textbf{88.1 $\pm$ 1.8} & \textbf{75.7 $\pm$ 4.1} & \textbf{75.2 $\pm$ 3.8} \\
    \bottomrule
  \end{tabular}
  }
  }
  \caption{Validation results on EntailmentBank (Task 2) of replacing the prover/verifier with oracles.}
  \label{table:entailmentbank_test_oracles}
\end{table*}

\section{Analyses}
\label{sec:analyses}

\subsection{Ablation Studies}

\smallsec{EntailmentBank}
Our full model searches for stepwise proofs, relying on both the verifier and the prover for producing scores. We conduct ablation studies on Task 2 of EntailmentBank to better understand the empirical gains coming from each of these components.

First, we compare the full model with the stepwise prover without search (Sec.~\ref{subsec:stepwise}). Results in Table~\ref{table:entailmentbank_ablation_test} show that the full model significantly improves upon this stepwise baseline across the board, demonstrating the benefits of searching for proofs at inference time.

Note that the stepwise baseline without search also performs significantly better than EntailmentWriter (31.8\% vs. 20.9\% in Overall-AllCorrect). We ask how much of the improvement is due to stepwise proof generation as opposed to implementation details and hyperparameters.\footnote{We do not have access to the exact details of EntailmentWriter since the training code has not been released.}. In Table~\ref{table:entailmentbank_ablation_test}, we replicate EntailmentWriter (\emph{Ours w/o search w/o stepwise}) in a unified implementation with our other methods. It outperforms EntailmentWriter (27.1\% vs. 20.9\% in Overall-AllCorrect) but still falls behind the stepwise baseline, demonstrating the effectiveness of stepwise proof generation.

Instead of averaging the scores from the prover and the verifier, what if we use only one of them? Can we still produce accurate and well-calibrated scores that are useful in proof search? In Table~\ref{table:entailmentbank_ablation_test}, we experiment with two additional versions of \namenospace: one without the prover score and the other without the verifier score. Results show that they fail to improve upon the stepwise baseline without search, demonstrating the necessity of combining the verifier and the prover.

Table~\ref{table:entailmentbank_ablation_test} also includes different methods' average inference time per test example, measured with batch size 1 and beam width 10. \name takes 4.4 seconds to process a test example, which is 2x slower than the stepwise baseline. Longer run time is a natural consequence of proof search, and 2x is a modest slow down.

\smallsec{RuleTaker}
We perform similar ablation experiments also on RuleTaker. Results in Table~\ref{table:ruletaker_ablation} show similar patterns as the ablations on EntailmentBank. However, the main difference is that proof search leads to a much larger improvement on RuleTaker, and the two baselines without search perform much lower than prior methods (ProofWriter in Table~\ref{table:ruletaker_main}). This is due to how we adapt proof generation systems to the task of RuleTaker.

As described in Sec.~\ref{subsec:ruletaker_adaptation}, the answer is produced by a linear classifier over proof scores of the hypothesis $h$ and its negation $\neg h$. To perform well, we need the proof generation system to (1) assign high scores to valid hypotheses and (2) assign low scores to invalid hypotheses. However, proof generation systems without verifiers---such as the two baselines---have never seen invalid proof steps in training. They are good at (1) but not necessarily (2); this is sufficient for EntailmentBank (with only valid hypotheses) but not RuleTaker. In contrast, proof generation systems with verifiers---such as our full model---are good at both (1) and (2). In other words, \name can generate more accurate proof scores for both valid and invalid hypotheses.

In addition, Table~\ref{table:ruletaker_ablation} shows that the verifier score alone is sufficient for RuleTaker; adding the prover score does not make much difference. This is because verifiers trained on RuleTaker are highly accurate, and they do not need to be supplemented by the prover.

\subsection{\name with Oracles}
The stepwise prover and the verifier are two major components in \namenospace. To analyze which one is the bottleneck, we construct ``oracle'' versions of them, both of which have access to ground-truth information for better predictions. Given a partial proof, the oracle prover generates multiple potential proof steps just like a regular prover. But it additionally includes all ground truth steps that are valid for the partial proof, i.e., steps whose premises have been satisfied by the partial proof. The oracle verifier also builds on a regular verifier but always assigns the highest score (1.0) to proof steps in the ground truth. Note that we call them ``oracles'', but neither of them is perfect. For example, the oracle verifier cannot reliably tell whether a proof step is valid if the step deviates even slightly from the ground truth.

Table~\ref{table:entailmentbank_test_oracles} shows the validation results on Task 2 of EntailmentBank. The oracle prover alone improves the performance significantly (e.g., a boost of 29.2 in Overall-AllCorrect), demonstrating that the prover is a major bottleneck. In contrast, the oracle verifier alone does not help much, improving only 0.8 in Overall-AllCorrect. However, 0.8 might underestimate the importance of the verifier, as the oracle verifier is not useful if the prover fails to generate the ground truth proof step in the first place. Actually, adding the oracle verifier to the oracle prover improves Overall-AllCorrect by 8.9, demonstrating that the verifier also bears room for improvement.

\begin{figure}[ht]
  \centering
  \includegraphics[width=1.0\linewidth]{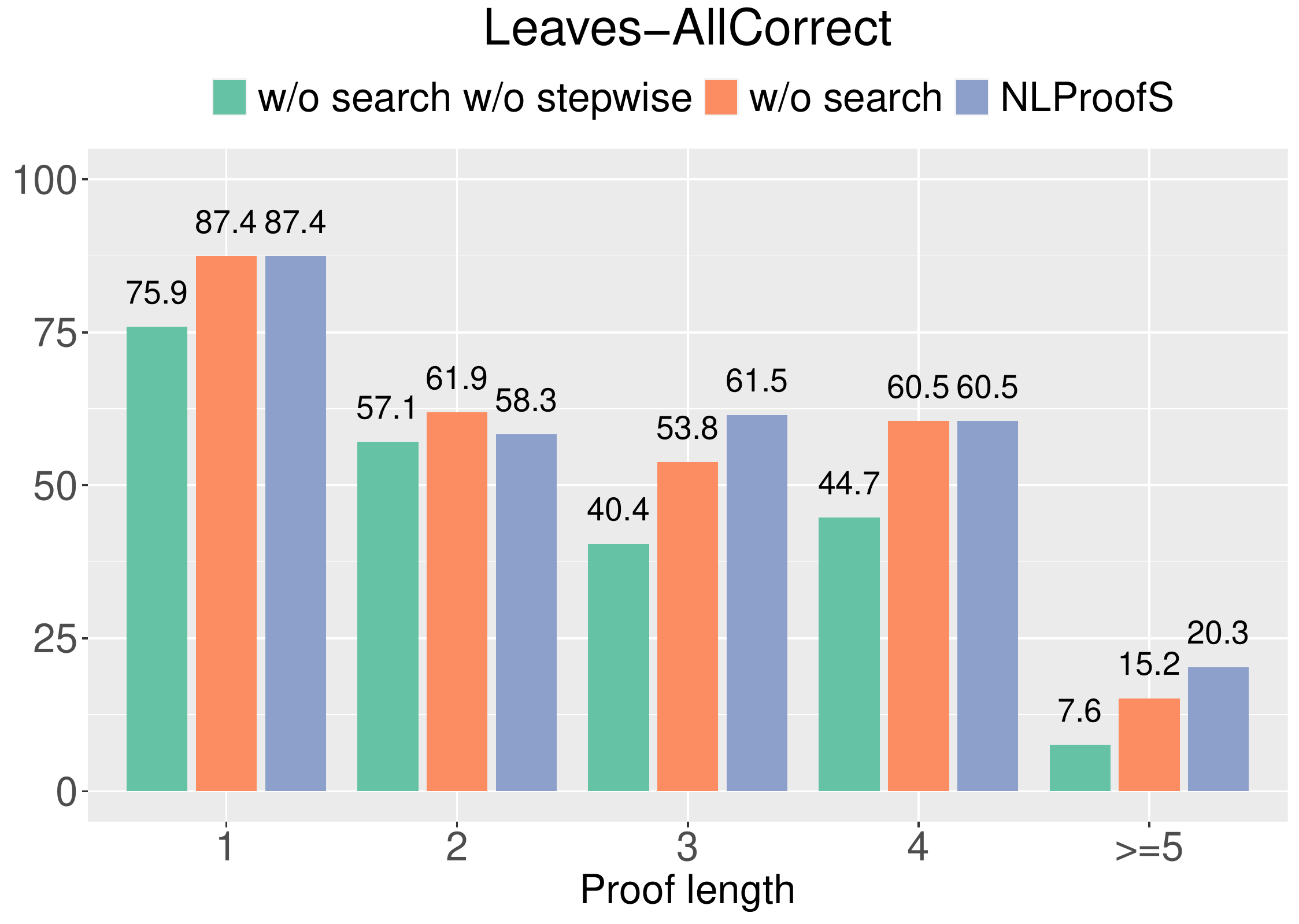}
  \caption{Test results on Task 2 (distractor) broken down by the length of the ground truth proof. Here we show the Leaves-AllCorrect metric.
  }
  \label{fig:proof_length_leaves}
\end{figure}

\subsection{Impact of Proof Length}
Prior work has demonstrated that proof generation models struggle with long proofs~\cite{dalvi2021explaining}. In Fig.~\ref{fig:proof_length_leaves}, we break down the test performance on Task 2 of EntailmentBank by proof length, i.e., the number of steps in the ground truth proof. Here we show only the Leaves-AllCorrect metric. Leaves metrics are relatively robust against the issue of multiple valid proof trees per example. Appendix~\ref{sec:additional_experiments} include figures of other metrics. But they may exaggerate the difficulty with long proofs, as the issue of multiple valid proofs is particularly prominent for long proofs. Nevertheless, we still see a significant performance drop in Fig.~\ref{fig:proof_length_leaves} when the proof length exceeds 1--2, suggesting that generating long proofs remains a challenge. However, we also see the benefits of proof search since its improvements over the stepwise baseline are more evident for long proofs.

In addition, \name tends to generate longer proofs compared to the baselines. On the validation set of Task 2, the ground truth proofs have an average length of 3.2 steps, whereas the average lengths of the generated proofs are 2.6, 2.7, and 2.9 for the single-shot baseline, the stepwise baseline, and \namenospace.

\subsection{Reduced Hallucination}
The verifier in \name aims to prevent the model from hallucinating invalid proof steps. However, it is difficult to evaluate hallucination automatically: when the model generation deviates from ground truth, it is difficult to evaluate whether it is a valid proof step. Therefore, besides qualitative examples in Table~\ref{table:examples}, we also perform a human evaluation similar to \citet{bostrom2022natural}.

We compare three models: EntailmentWriter, \name w/o search (our model without the verifier-guided search), and \name (our full model). For each model, we sample 100 generated proof steps and manually annotate them as valid/invalid. The percentage of valid steps is 43\%, 65\%, and 77\%, demonstrating the effectiveness of \name in mitigating hallucination.

\section{Conclusion}
We have introduced \name for stepwise proof generation in natural language. It learns to generate relevant proof steps conditioning on the hypothesis. To prevent hallucination, \name searches for proofs that maximize a validity score judged by a verifier. Our method has achieved state-of-the-art performance on EntailmentBank and RuleTaker, demonstrating the promise of stepwise proof generation for human-authored proofs. In the future, we hope to see increasing applications of verifiers and proof search in various reasoning tasks.

\clearpage
\section*{Limitations}

Despite the strong performance on two benchmarks, our method still has substantial room for future improvement. Currently, the prover (Sec.~\ref{subsec:stepwise}) uses beam search as the decoding algorithm, which has two problems: First, it generates equivalent proof steps such as ``\texttt{sent1 \& sent2 -> hypothesis}'' and ``\texttt{sent2 \& sent1 -> hypothesis}''. It would be more efficient if we make the decoding invariant to the permutation of premises. Second, the generated proof steps lack diversity. Since the verifier can filter out invalid proof steps, it is more important for the prover to have coverage and diversity than precision. It would be interesting to try more advanced decoding algorithms such as Diverse Beam Search~\citep{vijayakumar2016diverse}.
Like prior work~\citep{tafjord2021proofwriter,dalvi2021explaining}, our prover concatenates all supporting facts into a long text sequence and applies a Transformer encoder to it. This could be an inefficient use of computation and may have problems scaling to longer sentences or a larger number of supporting facts. Solutions like Fusion-in-Decoder~\cite{izacard2021leveraging} may help solve this problem.

\section*{Ethical Considerations}

Machine learning and NLP are moving from lab curiosity into real-world systems that make critical decisions in areas such as hiring, loan approval, and college admission. It is imperative that these decisions are interpretable to humans. Proof generation enhances interpretability by requiring the model to produce not only the final decision but also an explicit proof. However, the interpretability is jeopardized if the model learns to hallucinate invalid proof steps, like a person trying to find unfaithful excuses to justify a decision. Our method uses an independently trained verifier to check the validity of proof steps, which effectively reduces hallucination and enables the generated proof to explain the decision more faithfully.
\section*{Acknowledgements}
This work is partially supported by the Office of Naval Research under Grant N00014-20-1-2634. We gratefully acknowledge financial support from the Schmidt DataX Fund at Princeton University made possible through a major gift from the Schmidt Futures Foundation. We also thank Darby Haller, Jane Pan, Shunyu Yao, and the members of the Princeton NLP group for helpful discussion and valuable feedback.

\bibliographystyle{acl_natbib}
\bibliography{egbib}

\begin{thebibliography}{52}
\expandafter\ifx\csname natexlab\endcsname\relax\def\natexlab#1{#1}\fi

\bibitem[{Angeli et~al.(2016)Angeli, Nayak, and Manning}]{angeli2016combining}
Gabor Angeli, Neha Nayak, and Christopher~D Manning. 2016.
\newblock Combining natural logic and shallow reasoning for question answering.
\newblock In \emph{Annual Meeting of the Association for Computational
  Linguistics (ACL)}.

\bibitem[{Bostrom et~al.(2022)Bostrom, Sprague, Chaudhuri, and
  Durrett}]{bostrom2022natural}
Kaj Bostrom, Zayne Sprague, Swarat Chaudhuri, and Greg Durrett. 2022.
\newblock Natural language deduction through search over statement
  compositions.
\newblock \emph{arXiv preprint arXiv:2201.06028}.

\bibitem[{Bostrom et~al.(2021)Bostrom, Zhao, Chaudhuri, and
  Durrett}]{bostrom2021flexible}
Kaj Bostrom, Xinyu Zhao, Swarat Chaudhuri, and Greg Durrett. 2021.
\newblock Flexible generation of natural language deductions.
\newblock In \emph{Conference on Empirical Methods in Natural Language
  Processing (EMNLP)}.

\bibitem[{Bowman et~al.(2015)Bowman, Angeli, Potts, and
  Manning}]{bowman2015large}
Samuel Bowman, Gabor Angeli, Christopher Potts, and Christopher~D Manning.
  2015.
\newblock A large annotated corpus for learning natural language inference.
\newblock In \emph{Conference on Empirical Methods in Natural Language
  Processing (EMNLP)}.

\bibitem[{Brown et~al.(2020)Brown, Mann, Ryder, Subbiah, Kaplan, Dhariwal,
  Neelakantan, Shyam, Sastry, Askell et~al.}]{brown2020language}
Tom Brown, Benjamin Mann, Nick Ryder, Melanie Subbiah, Jared~D Kaplan, Prafulla
  Dhariwal, Arvind Neelakantan, Pranav Shyam, Girish Sastry, Amanda Askell,
  et~al. 2020.
\newblock Language models are few-shot learners.
\newblock In \emph{Advances in Neural Information Processing Systems
  (NeurIPS)}.

\bibitem[{Chen et~al.(2021)Chen, Tworek, Jun, Yuan, Pinto, Kaplan, Edwards,
  Burda, Joseph, Brockman et~al.}]{chen2021evaluating}
Mark Chen, Jerry Tworek, Heewoo Jun, Qiming Yuan, Henrique Ponde de~Oliveira
  Pinto, Jared Kaplan, Harri Edwards, Yuri Burda, Nicholas Joseph, Greg
  Brockman, et~al. 2021.
\newblock Evaluating large language models trained on code.
\newblock \emph{arXiv preprint arXiv:2107.03374}.

\bibitem[{Clark et~al.(2020)Clark, Tafjord, and
  Richardson}]{clark2021transformers}
Peter Clark, Oyvind Tafjord, and Kyle Richardson. 2020.
\newblock Transformers as soft reasoners over language.
\newblock In \emph{International Joint Conference on Artificial Intelligence
  (IJCAI)}.

\bibitem[{Cobbe et~al.(2021)Cobbe, Kosaraju, Bavarian, Hilton, Nakano, Hesse,
  and Schulman}]{cobbe2021training}
Karl Cobbe, Vineet Kosaraju, Mohammad Bavarian, Jacob Hilton, Reiichiro Nakano,
  Christopher Hesse, and John Schulman. 2021.
\newblock Training verifiers to solve math word problems.
\newblock \emph{arXiv preprint arXiv:2110.14168}.

\bibitem[{Dalvi et~al.(2021)Dalvi, Jansen, Tafjord, Xie, Smith, Pipatanangkura,
  and Clark}]{dalvi2021explaining}
Bhavana Dalvi, Peter Jansen, Oyvind Tafjord, Zhengnan Xie, Hannah Smith,
  Leighanna Pipatanangkura, and Peter Clark. 2021.
\newblock Explaining answers with entailment trees.
\newblock \emph{arXiv preprint arXiv:2104.08661}.

\bibitem[{Dalvi et~al.(2022)Dalvi, Tafjord, and Clark}]{dalvi2022towards}
Bhavana Dalvi, Oyvind Tafjord, and Peter Clark. 2022.
\newblock Towards teachable reasoning systems.
\newblock \emph{arXiv preprint arXiv:2204.13074}.

\bibitem[{Gontier et~al.(2020)Gontier, Sinha, Reddy, and
  Pal}]{gontier2020measuring}
Nicolas Gontier, Koustuv Sinha, Siva Reddy, and Chris Pal. 2020.
\newblock Measuring systematic generalization in neural proof generation with
  transformers.
\newblock In \emph{Advances in Neural Information Processing Systems
  (NeurIPS)}.

\bibitem[{Hong et~al.(2022)Hong, Zhang, Yu, and Zhang}]{hong2022metgen}
Ruixin Hong, Hongming Zhang, Xintong Yu, and Changshui Zhang. 2022.
\newblock {MetGen}: A module-based entailment tree generation framework for
  answer explanation.
\newblock In \emph{Findings of the North American Chapter of the Association
  for Computational Linguistics: NAACL}.

\bibitem[{Izacard and Grave(2021)}]{izacard2021leveraging}
Gautier Izacard and {\'E}douard Grave. 2021.
\newblock Leveraging passage retrieval with generative models for open domain
  question answering.
\newblock In \emph{European Chapter of the Association for Computational
  Linguistics (EACL)}, pages 874--880.

\bibitem[{Jiang et~al.(2020)Jiang, Bordia, Zhong, Dognin, Singh, and
  Bansal}]{jiang2020hover}
Yichen Jiang, Shikha Bordia, Zheng Zhong, Charles Dognin, Maneesh Singh, and
  Mohit Bansal. 2020.
\newblock {HoVer}: A dataset for many-hop fact extraction and claim
  verification.
\newblock In \emph{Findings of the Association for Computational Linguistics:
  EMNLP}.

\bibitem[{Jiang et~al.(2019)Jiang, Joshi, Chen, and Bansal}]{jiang2019explore}
Yichen Jiang, Nitish Joshi, Yen-Chun Chen, and Mohit Bansal. 2019.
\newblock Explore, propose, and assemble: An interpretable model for multi-hop
  reading comprehension.
\newblock In \emph{Annual Meeting of the Association for Computational
  Linguistics (ACL)}.

\bibitem[{Kalyanpur et~al.(2022)Kalyanpur, Breloff, Ferrucci, Lally, and
  Jantos}]{kalyanpur2020braid}
Aditya Kalyanpur, Tom Breloff, David Ferrucci, Adam Lally, and John Jantos.
  2022.
\newblock Braid: Weaving symbolic and neural knowledge into coherent logical
  explanations.
\newblock In \emph{AAAI Conference on Artificial Intelligence}.

\bibitem[{Kathryn and Mazaitis(2018)}]{kathryn2018tensorlog}
William W Cohen Fan~Yang Kathryn and Rivard Mazaitis. 2018.
\newblock {TensorLog}: Deep learning meets probabilistic databases.
\newblock \emph{Journal of Artificial Intelligence Research}, 1:1--15.

\bibitem[{Kojima et~al.(2022)Kojima, Gu, Reid, Matsuo, and
  Iwasawa}]{kojima2022large}
Takeshi Kojima, Shixiang~Shane Gu, Machel Reid, Yutaka Matsuo, and Yusuke
  Iwasawa. 2022.
\newblock Large language models are zero-shot reasoners.
\newblock In \emph{Advances in Neural Information Processing Systems
  (NeurIPS)}.

\bibitem[{Kov{\'a}cs and Voronkov(2013)}]{kovacs2013first}
Laura Kov{\'a}cs and Andrei Voronkov. 2013.
\newblock First-order theorem proving and {Vampire}.
\newblock In \emph{International Conference on Computer Aided Verification
  (CAV)}.

\bibitem[{Lai et~al.(2017)Lai, Bisk, and Hockenmaier}]{lai2017natural}
Alice Lai, Yonatan Bisk, and Julia Hockenmaier. 2017.
\newblock Natural language inference from multiple premises.
\newblock In \emph{International Joint Conference on Natural Language
  Processing (IJCNLP)}.

\bibitem[{Lee et~al.(2016)Lee, He, Yih, Gao, Deng, and
  Smolensky}]{lee2015reasoning}
Moontae Lee, Xiaodong He, Wen-tau Yih, Jianfeng Gao, Li~Deng, and Paul
  Smolensky. 2016.
\newblock Reasoning in vector space: An exploratory study of question
  answering.
\newblock In \emph{International Conference on Learning Representations
  (ICLR)}.

\bibitem[{Liang et~al.(2021)Liang, Bethard, and
  Surdeanu}]{liang2021explainable}
Zhengzhong Liang, Steven Bethard, and Mihai Surdeanu. 2021.
\newblock Explainable multi-hop verbal reasoning through internal monologue.
\newblock In \emph{Annual Conference of the North American Chapter of the
  Association for Computational Linguistics (NAACL)}.

\bibitem[{Liu et~al.(2019)Liu, Ott, Goyal, Du, Joshi, Chen, Levy, Lewis,
  Zettlemoyer, and Stoyanov}]{liu2019roberta}
Yinhan Liu, Myle Ott, Naman Goyal, Jingfei Du, Mandar Joshi, Danqi Chen, Omer
  Levy, Mike Lewis, Luke Zettlemoyer, and Veselin Stoyanov. 2019.
\newblock {RoBERTa}: A robustly optimized {BERT} pretraining approach.
\newblock \emph{arXiv preprint arXiv:1907.11692}.

\bibitem[{Loshchilov and Hutter(2019)}]{loshchilov2018decoupled}
Ilya Loshchilov and Frank Hutter. 2019.
\newblock Decoupled weight decay regularization.
\newblock In \emph{International Conference on Learning Representations
  (ICLR)}.

\bibitem[{Madaan et~al.(2022)Madaan, Zhou, Alon, Yang, and
  Neubig}]{madaan2022language}
Aman Madaan, Shuyan Zhou, Uri Alon, Yiming Yang, and Graham Neubig. 2022.
\newblock Language models of code are few-shot commonsense learners.
\newblock \emph{arXiv preprint arXiv:2210.07128}.

\bibitem[{McCarthy et~al.(1960)}]{mccarthy1960programs}
John McCarthy et~al. 1960.
\newblock \emph{Programs with common sense}.
\newblock RLE and MIT Computation Center.

\bibitem[{Min et~al.(2019)Min, Zhong, Zettlemoyer, and
  Hajishirzi}]{min2019multi}
Sewon Min, Victor Zhong, Luke Zettlemoyer, and Hannaneh Hajishirzi. 2019.
\newblock Multi-hop reading comprehension through question decomposition and
  rescoring.
\newblock In \emph{Annual Meeting of the Association for Computational
  Linguistics (ACL)}.

\bibitem[{Mineshima et~al.(2015)Mineshima, Mart{\'\i}nez-G{\'o}mez, Miyao, and
  Bekki}]{mineshima2015higher}
Koji Mineshima, Pascual Mart{\'\i}nez-G{\'o}mez, Yusuke Miyao, and Daisuke
  Bekki. 2015.
\newblock Higher-order logical inference with compositional semantics.
\newblock In \emph{Conference on Empirical Methods in Natural Language
  Processing (EMNLP)}.

\bibitem[{Polu and Sutskever(2020)}]{polu2020generative}
Stanislas Polu and Ilya Sutskever. 2020.
\newblock Generative language modeling for automated theorem proving.
\newblock \emph{arXiv preprint arXiv:2009.03393}.

\bibitem[{Qu et~al.(2022)Qu, Cao, Gao, Ding, and Xu}]{qu2022interpretable}
Hanhao Qu, Yu~Cao, Jun Gao, Liang Ding, and Ruifeng Xu. 2022.
\newblock Interpretable proof generation via iterative backward reasoning.
\newblock In \emph{Annual Conference of the North American Chapter of the
  Association for Computational Linguistics (NAACL)}.

\bibitem[{Rae et~al.(2021)Rae, Borgeaud, Cai, Millican, Hoffmann, Song,
  Aslanides, Henderson, Ring, Young et~al.}]{rae2021scaling}
Jack~W Rae, Sebastian Borgeaud, Trevor Cai, Katie Millican, Jordan Hoffmann,
  Francis Song, John Aslanides, Sarah Henderson, Roman Ring, Susannah Young,
  et~al. 2021.
\newblock Scaling language models: Methods, analysis \& insights from training
  {Gopher}.
\newblock \emph{arXiv preprint arXiv:2112.11446}.

\bibitem[{Raffel et~al.(2020)Raffel, Shazeer, Roberts, Lee, Narang, Matena,
  Zhou, Li, and Liu}]{raffel2020exploring}
Colin Raffel, Noam Shazeer, Adam Roberts, Katherine Lee, Sharan Narang, Michael
  Matena, Yanqi Zhou, Wei Li, and Peter~J Liu. 2020.
\newblock Exploring the limits of transfer learning with a unified text-to-text
  transformer.
\newblock \emph{Journal of Machine Learning Research (JMLR)}, 21:1--67.

\bibitem[{Ribeiro et~al.(2022)Ribeiro, Wang, Ma, Dong, Wei, Zhu, Chen, Huang,
  Xu, Arnold et~al.}]{ribeiro2022entailment}
Danilo Ribeiro, Shen Wang, Xiaofei Ma, Rui Dong, Xiaokai Wei, Henry Zhu, Xinchi
  Chen, Zhiheng Huang, Peng Xu, Andrew Arnold, et~al. 2022.
\newblock Entailment tree explanations via iterative retrieval-generation
  reasoner.
\newblock In \emph{Findings of the North American Chapter of the Association
  for Computational Linguistics: NAACL}.

\bibitem[{Robertson et~al.(2009)Robertson, Zaragoza
  et~al.}]{robertson2009probabilistic}
Stephen Robertson, Hugo Zaragoza, et~al. 2009.
\newblock The probabilistic relevance framework: {BM25} and beyond.
\newblock \emph{Foundations and Trends{\textregistered} in Information
  Retrieval}, 3(4):333--389.

\bibitem[{Robinson and Voronkov(2001)}]{robinson2001handbook}
Alan~JA Robinson and Andrei Voronkov. 2001.
\newblock \emph{Handbook of automated reasoning}, volume~1.
\newblock Elsevier.

\bibitem[{Ruis et~al.(2020)Ruis, Andreas, Baroni, Bouchacourt, and
  Lake}]{ruis2020benchmark}
Laura Ruis, Jacob Andreas, Marco Baroni, Diane Bouchacourt, and Brenden~M Lake.
  2020.
\newblock A benchmark for systematic generalization in grounded language
  understanding.
\newblock In \emph{Advances in Neural Information Processing Systems
  (NeurIPS)}.

\bibitem[{Russell and Norvig(2002)}]{russell2002artificial}
Stuart Russell and Peter Norvig. 2002.
\newblock \emph{Artificial intelligence: a modern approach}.

\bibitem[{Saha et~al.(2020)Saha, Ghosh, Srivastava, and
  Bansal}]{saha2020prover}
Swarnadeep Saha, Sayan Ghosh, Shashank Srivastava, and Mohit Bansal. 2020.
\newblock {PRover}: Proof generation for interpretable reasoning over rules.
\newblock In \emph{Conference on Empirical Methods in Natural Language
  Processing (EMNLP)}.

\bibitem[{Sanyal et~al.(2022)Sanyal, Singh, and Ren}]{sanyal2022fairr}
Soumya Sanyal, Harman Singh, and Xiang Ren. 2022.
\newblock {FaiRR}: Faithful and robust deductive reasoning over natural
  language.
\newblock In \emph{Annual Meeting of the Association for Computational
  Linguistics (ACL)}.

\bibitem[{Saparov and Mitchell(2022)}]{saparov2021generative}
Abulhair Saparov and Tom~M Mitchell. 2022.
\newblock Towards general natural language understanding with probabilistic
  worldbuilding.
\newblock \emph{Transactions of the Association for Computational Linguistics
  (TACL)}, 10:325--342.

\bibitem[{Sellam et~al.(2020)Sellam, Das, and Parikh}]{sellam2020bleurt}
Thibault Sellam, Dipanjan Das, and Ankur Parikh. 2020.
\newblock {BLEURT}: Learning robust metrics for text generation.
\newblock In \emph{Annual Meeting of the Association for Computational
  Linguistics (ACL)}.

\bibitem[{Sinha et~al.(2019)Sinha, Sodhani, Dong, Pineau, and
  Hamilton}]{sinha2019clutrr}
Koustuv Sinha, Shagun Sodhani, Jin Dong, Joelle Pineau, and William~L Hamilton.
  2019.
\newblock {CLUTRR}: A diagnostic benchmark for inductive reasoning from text.
\newblock In \emph{Conference on Empirical Methods in Natural Language
  Processing (EMNLP)}.

\bibitem[{Smolensky(1990)}]{smolensky1990tensor}
Paul Smolensky. 1990.
\newblock Tensor product variable binding and the representation of symbolic
  structures in connectionist systems.
\newblock \emph{Artificial Intelligence}, 46:159--216.

\bibitem[{Sun et~al.(2021)Sun, Zhang, Chen, Gan, Wu, Chen, Zhou, and
  Li}]{sun2021probabilistic}
Changzhi Sun, Xinbo Zhang, Jiangjie Chen, Chun Gan, Yuanbin Wu, Jiaze Chen, Hao
  Zhou, and Lei Li. 2021.
\newblock Probabilistic graph reasoning for natural proof generation.
\newblock In \emph{Findings of the Association for Computational Linguistics:
  ACL}.

\bibitem[{Tafjord et~al.(2021)Tafjord, Dalvi, and
  Clark}]{tafjord2021proofwriter}
Oyvind Tafjord, Bhavana Dalvi, and Peter Clark. 2021.
\newblock {ProofWriter}: Generating implications, proofs, and abductive
  statements over natural language.
\newblock In \emph{Findings of the Association for Computational Linguistics:
  ACL}.

\bibitem[{Vijayakumar et~al.(2018)Vijayakumar, Cogswell, Selvaraju, Sun, Lee,
  Crandall, and Batra}]{vijayakumar2016diverse}
Ashwin~K Vijayakumar, Michael Cogswell, Ramprasath~R Selvaraju, Qing Sun,
  Stefan Lee, David Crandall, and Dhruv Batra. 2018.
\newblock {Diverse Beam Search}: Decoding diverse solutions from neural
  sequence models.
\newblock In \emph{AAAI Conference on Artificial Intelligence}.

\bibitem[{Weber et~al.(2019)Weber, Minervini, M{\"u}nchmeyer, Leser, and
  Rockt{\"a}schel}]{weber2019nlprolog}
Leon Weber, Pasquale Minervini, Jannes M{\"u}nchmeyer, Ulf Leser, and Tim
  Rockt{\"a}schel. 2019.
\newblock {NLProlog}: Reasoning with weak unification for question answering in
  natural language.
\newblock In \emph{Annual Meeting of the Association for Computational
  Linguistics (ACL)}.

\bibitem[{Wei et~al.(2022)Wei, Wang, Schuurmans, Bosma, Chi, Le, and
  Zhou}]{wei2022chain}
Jason Wei, Xuezhi Wang, Dale Schuurmans, Maarten Bosma, Ed~Chi, Quoc Le, and
  Denny Zhou. 2022.
\newblock Chain of thought prompting elicits reasoning in large language
  models.
\newblock In \emph{Advances in Neural Information Processing Systems
  (NeurIPS)}.

\bibitem[{Xie et~al.(2020)Xie, Thiem, Martin, Wainwright, Marmorstein, and
  Jansen}]{xie2020worldtree}
Zhengnan Xie, Sebastian Thiem, Jaycie Martin, Elizabeth Wainwright, Steven
  Marmorstein, and Peter Jansen. 2020.
\newblock {WorldTree V2}: A corpus of science-domain structured explanations
  and inference patterns supporting multi-hop inference.
\newblock In \emph{International Conference on Language Resources and
  Evaluation (LREC)}.

\bibitem[{Yang and Deng(2019)}]{yang2019learning}
Kaiyu Yang and Jia Deng. 2019.
\newblock Learning to prove theorems via interacting with proof assistants.
\newblock In \emph{International Conference on Machine Learning (ICML)}.

\bibitem[{Yang and Deng(2021)}]{yang2021learning}
Kaiyu Yang and Jia Deng. 2021.
\newblock Learning symbolic rules for reasoning in quasi-natural language.
\newblock \emph{arXiv preprint arXiv:2111.12038}.

\bibitem[{Yang et~al.(2018)Yang, Qi, Zhang, Bengio, Cohen, Salakhutdinov, and
  Manning}]{yang2018hotpotqa}
Zhilin Yang, Peng Qi, Saizheng Zhang, Yoshua Bengio, William Cohen, Ruslan
  Salakhutdinov, and Christopher~D Manning. 2018.
\newblock {HotpotQA}: A dataset for diverse, explainable multi-hop question
  answering.
\newblock In \emph{Conference on Empirical Methods in Natural Language
  Processing (EMNLP)}.

\end{thebibliography}

\appendix
\clearpage

\setcounter{table}{0}
\renewcommand{\thetable}{\Alph{table}}
\setcounter{figure}{0}
\renewcommand{\thefigure}{\Alph{figure}}

\begin{table*}[ht]
  \centering 
  \makebox[1 \textwidth][c]{
  \resizebox{1 \textwidth}{!}{
  \begin{tabular}{@{}llllllll@{}}
    \toprule
     Method & \multicolumn{2}{c}{Leaves} & \multicolumn{2}{c}{Steps} & \multicolumn{2}{c}{Intermediates} & Overall \\
     \cmidrule(r){2-3} \cmidrule(r){4-5} \cmidrule(r){6-7} & F1 & AllCorrect & F1 & AllCorrect & F1 & AllCorrect & AllCorrect  \\
    \midrule
    Our format & \underline{86.9 $\pm$ 0.6} & \underline{45.6 $\pm$ 1.5} & \textbf{42.6 $\pm$ 1.6} & \textbf{29.7 $\pm$ 1.3} & \textbf{64.6 $\pm$ 1.4} & \underline{32.2 $\pm$ 2.1} & \underline{27.1 $\pm$ 1.5} \\
    EntailmentWriter format & \textbf{87.6 $\pm$ 0.5} & \textbf{47.1 $\pm$ 2.1} & \underline{42.2 $\pm$ 1.1} &  \textbf{29.7 $\pm$ 1.6} & \underline{64.5 $\pm$ 0.6} & \textbf{32.3 $\pm$ 1.7}	& \textbf{27.5 $\pm$ 1.8} \\
    \bottomrule
  \end{tabular}
  }
  }
  \caption{Test results of single-shot models on EntailmentBank~\citep{dalvi2021explaining} (Task 2) with different input formats. All methods are based on T5-large~\cite{raffel2020exploring}.}
  \label{table:entailmentbank_format}
\end{table*}

\begin{figure*}[h]
  \centering
  \includegraphics[width=1.0\linewidth]{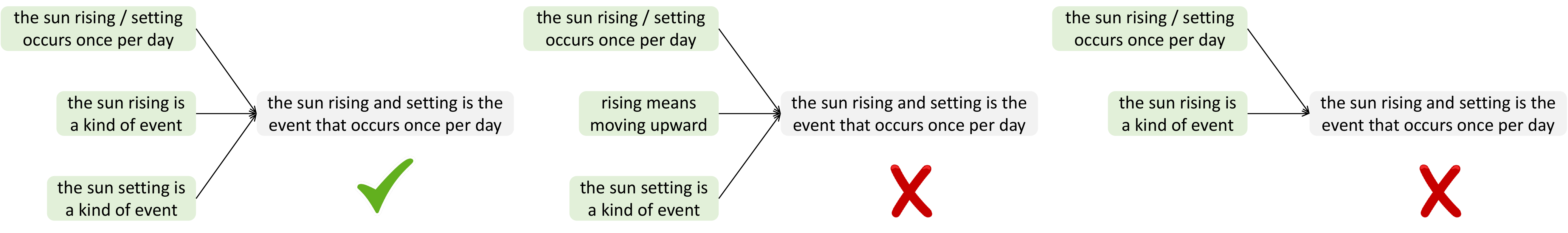}
  \caption{Pseudo-negative examples constructed by perturbing positive examples.}
  \label{fig:pseudo_negative_1}
\end{figure*}

\begin{figure}[h]
  \centering
  \includegraphics[width=0.6\linewidth]{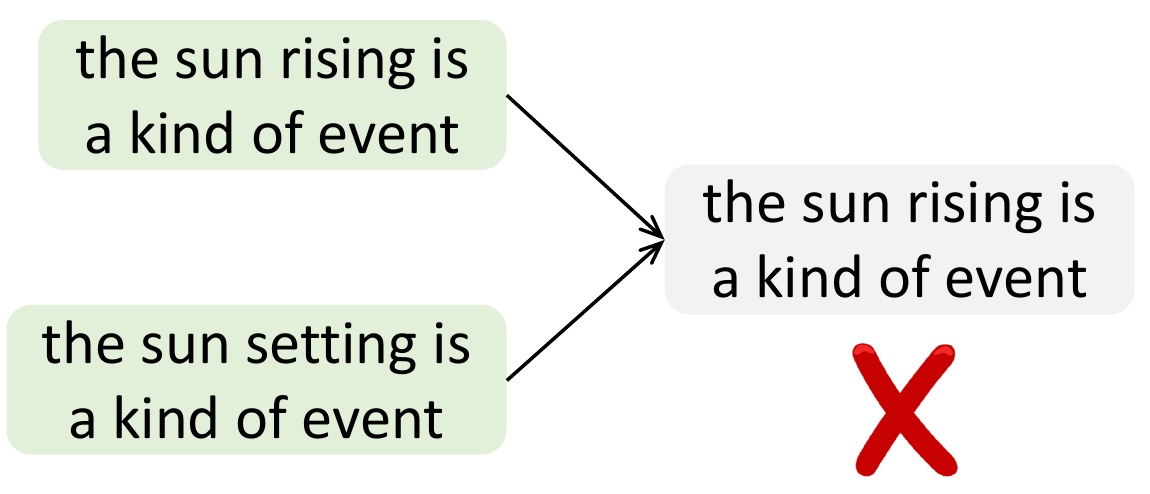}
  \caption{EntailmentBank pseudo-negative examples constructed by copying premises. The step is technically valid but not useful.}
  \label{fig:pseudo_negative_2}
\end{figure}

\begin{figure}[h]
  \centering
  \includegraphics[width=1.0\linewidth]{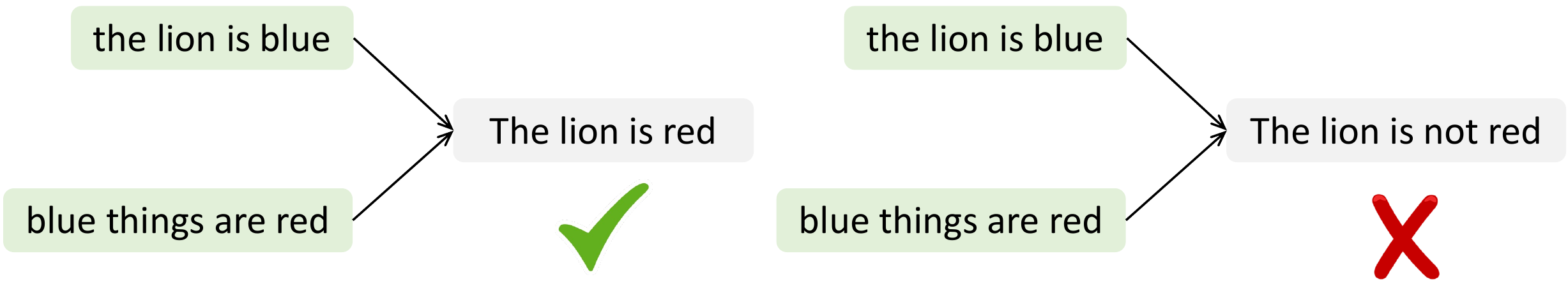}
  \caption{RuleTaker pseudo-negative examples constructed by negating the conclusion.}
  \label{fig:pseudo_negative_3}
\end{figure}

\section{Evaluation Metrics on EntailmentBank}
\label{sec:entailmentbank_metrics}

We evaluate on EntailmentBank using their official evaluation metrics calculated by their evaluation code. Below is a summary; please refer to the EntailmentBank paper for further details.

Let $\hat{T}$ be a generated proof tree, with $T$ being the ground truth. First, nodes in $\hat{T}$ are aligned with nodes in $T$ using a tree alignment algorithm based on the ``sent*'' labels. Once aligned, it is scored using 4 types of metrics---Leaves, Steps, Intermediates, and Overall.
\begin{itemize}[leftmargin=*]
\setlength\itemsep{-0.1em}
\item \emph{Leaves (F1, AllCorrect)}: The Leaves metrics compare the leaf nodes of $\hat{T}$ and $T$ to calculate an F1 score and an ``AllCorrect'' score, which means all predicted nodes are correct. In other words, AllCorrect = 1 if F1 = 1, and AllCorrect = 0 if F1 < 1.
\item \emph{Steps (F1, AllCorrect)}: The Steps metrics measure whether predicted proof steps are structurally correct. A predicted step corresponds to an internal node $u \in \hat{T}$ (aligned to $v \in T$). It is structurally correct if the children of $u$ and $v$ are also perfectly aligned. Since there are multiple steps in $\hat{T}$ and $T$, we can calculate F1 and AllCorrect.
\item \emph{Intermediates (F1, AllCorrect)}: An intermediate conclusion $u \in \hat{T}$ (aligned to $v \in T$) is correct if the BLEURT~\cite{sellam2020bleurt}\footnote{We use the \texttt{
bleurt-large-512} model following \citet{dalvi2021explaining}.} score between $u$ and $v$ is greater than 0.28. We calculate F1 and AllCorrect from all intermediate conclusions in $\hat{T}$ and $T$.
\item \emph{Overall (AllCorrect)}: The Overall metric evaluates whether the leaves, steps, and intermediates are all correct, i.e., AllCorrect = 1 if and only if $\hat{T}$ matches completely with $T$.
\end{itemize}

\section{Different Input Formats}
\label{sec:entailmentbank_format}

We use a slightly different input format from EntailmentWriter~\citep{dalvi2021explaining}, as their format had not been released when we developed our method.

Consider the example in Fig.~\ref{fig:task}. The input to our single-shot baseline (\emph{\name w/o search w/o stepwise} in Table~\ref{table:entailmentbank_ablation_test}) is ``\texttt{\$hypothesis\$ = solar is a kind of renewable energy for heating homes ;
\$context\$ = sent1: homes are buildings sent2: solar is renewable \dots ;}'', whereas the their input is ``\texttt{\$proof\$ ; \$question\$ = As a kind of renewable energy, what can solar be used for? ; \$answer\$ = heating homes ; \$hypothesis\$ = solar is a kind of renewable energy for heating homes ;
\$context\$ = sent1: homes are buildings sent2: solar is renewable \dots ;}'', which includes more information (\texttt{\$question\$} and \texttt{\$answer\$}) than ours.

We experiment with single-shot methods implemented in our codebase using their input format. Results in Table~\ref{table:entailmentbank_format} indicate no significant difference.

\section{Pseudo-negative Examples for Training the Verifier}
\label{sec:negatives}

As mentioned in Sec.~\ref{subsec:verifier}, the negative examples used for training the verifier are constructed automatically using the procedure below:
\begin{itemize}[leftmargin=*]
\setlength\itemsep{-0.1em}
\item As in Fig.~\ref{fig:pseudo_negative_1}, for each positive example consisting of premises and a conclusion, we either remove some premises or replacing one premise with a distractor retrieved from $C$ using BM25~\citep{robertson2009probabilistic}.
\item For EntailmentBank, as in Fig.~\ref{fig:pseudo_negative_2}, we generate additional pseudo-negatives by copying one of the premises as the conclusion.
\item For RuleTaker, as in Fig.~\ref{fig:pseudo_negative_3}, we generate additional pseudo-negatives by negating the conclusion.
\end{itemize}

\section{Proof Graphs are Loopless}
\label{sec:tree}

We prove that the proof graph in Algorithm~\ref{alg:search} is loopless. Intuitively, for a proof graph $(C, I, S, h)$, as we traverse along any path, the node scores in $C \cup I \cup \{h\}$ are non-increasing due to Eqn.~\ref{eqn:proof_score}, which prevents loops.

\begin{lemma}
\label{thm:single_step}
Let $\mathcal{G}$ be a proof graph with nodes $(C, I, S, h)$. For any $s \in S$ and  $v, u \in C \cup I \cup \{h\}$ s.t. edges $(v, s)$ and $(s, u)$ exist, we have $\texttt{scr}_n(u) \leq \texttt{scr}_n(v)$.
\end{lemma}
\begin{proof}
In this case, $s$ must be a proof step with $u$ as the conclusion and $v$ as one of its premises. According to Definition~\ref{dfn:proof_graph} and Eqn.~\ref{eqn:proof_score}, we have
\begin{equation*}
\texttt{scr}_n(u) = \mathrm{min}\big(\texttt{scr}_s(s), \texttt{scr}_n(v), \dots \big).
\end{equation*}
Therefore, $\texttt{scr}_n(u) \leq \texttt{scr}_n(v)$.
\end{proof}

\begin{lemma}
\label{thm:multi_step}
Let $\mathcal{G}$ be a proof graph with nodes $(C, I, S, h)$. For any $v, u \in C \cup I \cup \{h\}$ s.t. $v$ is a predecessor of $u$, we have $\texttt{scr}_n(u) \leq \texttt{scr}_n(v)$.
\end{lemma}
\begin{proof}
According to Definition~\ref{dfn:proof_graph}, there exists a path $v \rightarrow s_1 \rightarrow w_1 \rightarrow s_2 \rightarrow w_2 \rightarrow, \dots, \rightarrow s_k \rightarrow u$, where $\forall i, s_i \in S, w_i \in C \cup I \cup \{h\}$. The lemma can be proved by performing induction on the path and applying Lemma~\ref{thm:single_step}.
\end{proof}

\begin{theorem}
In Algorithm~\ref{alg:search}, if the proof graph is loopless after initialization, then it will remain loopless.
\end{theorem}
\begin{proof}
We just need to prove that it is impossible to introduce a loop during any iteration in Algorithm~\ref{alg:search}. We prove it by contradiction, assuming we could introduce a loop in an iteration, as in Fig.~\ref{fig:loop}. We have two nodes $v, u \in C \cup I \cup \{h\}$, and $v$ is a predecessor of $u$ before introducing the loop. Further assume that the loop is introduced as a result of executing a proof step $s$, which created the edges $(u, s)$ and $(s, v)$ (the blue arrow in Fig.~\ref{fig:loop}; $s$ is omitted). In this hypothetical scenario, the loop would be $v \rightarrow \dots \rightarrow u \rightarrow s \rightarrow v$.

Apply Lemma~\ref{thm:multi_step} to the path from $v$ to $u$, and we have $\texttt{scr}_n(u) \leq \texttt{scr}_n(v)$ before introducing the loop. Remember how the step $s$ is executed (Sec.~\ref{subsec:proof_search}), the tentative score $\widehat{\texttt{scr}}_n(v) = \mathrm{min}\big(\texttt{scr}_s(s), \texttt{scr}_n(u), \dots \big) \leq \texttt{scr}_n(u) \leq \texttt{scr}_n(v)$. The tentative score is not greater than the original score of $v$. So the step is a no-op that should not be executed. Therefore, it is impossible to introduce loops.
\end{proof}

\begin{figure}[ht]
  \centering
  \includegraphics[width=1.0\linewidth]{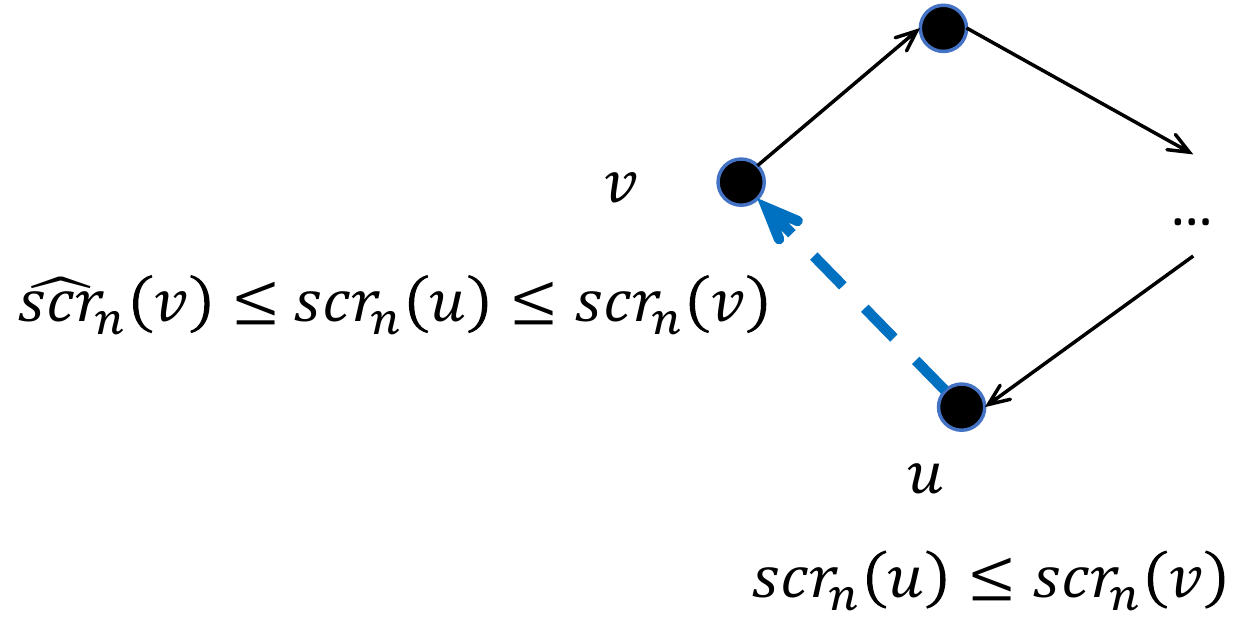}
  \caption{A hypothetical loop in the proof graph. Nodes for proof step ($S$) are omitted for simplicity of illustration. The blue arrow is a hypothetical proof step that introduces the loop, which should actually be a no-op because the tentative score  $\widehat{\texttt{scr}}_n(v)$ not greater than the existing score  $\texttt{scr}_n(v)$ (Sec.~\ref{subsec:proof_search}).}
  \label{fig:loop}
\end{figure}

\section{Procedure for Sampling Partial Proofs}
\label{sec:sampling}

The \texttt{sample\_new} function in Algorithm~\ref{alg:search} samples a partial proof tree from the proof graph. First, the graph is a DAG, so we can visit nodes in the order of a topological sort---successors before predecessors. Second, when visiting a node, if it is not already a part of the partial proof, we add it with a probability of 0.5. Third, whenever we add a node, we also add all of its predecessors. This ensures the result is a valid partial proof.

\section{Training and inference details}
\label{sec:details}

We use T5-large~\cite{raffel2020exploring} for the prover and RoBERTa-large~\cite{liu2019roberta} for the verifier.
All experiments are run on machines with 2 CPUs, 16GB memory, and one NVIDIA A6000 GPU. Models are optimized using AdamW~\cite{loshchilov2018decoupled}. The learning rate warms up linearly from 0 to a maximum value and then decays following the cosine schedule. Hyperparameters are tuned on the validation data separately for each task/method. We report test results of models trained on the training set alone, excluding the validation set. We report the average performance and the standard deviation for 5 independent runs.

Our results on EntailmentBank are produced by the official evaluation code.\footnote{\url{https://github.com/allenai/entailment_bank}} The code had a bug fix in May 2022 that impacted the Intermediate-AllCorrect metric of methods evaluated earlier, including IRGR~\citep{ribeiro2022entailment} and arXiv versions v1, v2 of EntailmentWriter~\citep{dalvi2021explaining}. We evaluate NLProofS using the evaluation code after the bug fix. And we report the EntailmentBank numbers based on their fixed arXiv version v3 that was released on May 28, 2022. The numbers in the IRGR paper have not been updated yet, so we report its Intermediates-AllCorrect metric based on private correspondence with the authors.

%!TEX root = ../paper.tex

\begin{table*}[ht]
  \centering
  \makebox[1 \textwidth][c]{
  \resizebox{1 \textwidth}{!}{
  \begin{tabular}{@{}llllllll@{}}
    \toprule
     Method & \multicolumn{2}{c}{Leaves} & \multicolumn{2}{c}{Steps} & \multicolumn{2}{c}{Intermediates} & Overall \\
     \cmidrule(r){2-3} \cmidrule(r){4-5} \cmidrule(r){6-7} & F1 & AllCorrect & F1 & AllCorrect & F1 & AllCorrect & AllCorrect  \\
    \midrule
    EntailmentWriter & 86.2 & 43.9 & 40.6 & 28.3 & 67.1 & 34.8 & 27.3 \\
    EntailmentWriter (T5-11B) & \textbf{89.4} & 52.9 & 46.6 & 35.3 & 69.1 & 36.9 & 32.1 \\
    \name (ours) & \textbf{89.4 $\pm$ 0.8} & \textbf{56.0 $\pm$ 0.7} & \textbf{50.4 $\pm$ 1.9} & \textbf{38.4 $\pm$ 1.3} & \textbf{71.9 $\pm$ 1.4} & \textbf{41.3 $\pm$ 1.4} & \textbf{37.1 $\pm$ 1.5} \\
    \midrule
    GPT-3~\citep{brown2020language} & 64.2 $\pm$ 2.3 & 15.3 $\pm$ 1.9 & 17.6 $\pm$ 0.6 & 12.3 $\pm$ 1.4 & 53.6 $\pm$ 1.4 & 22.3 $\pm$ 1.1 & 12.3 $\pm$ 1.4 \\
    Codex~\citep{chen2021evaluating} & 68.9 $\pm$ 3.7 & 19.8 $\pm$ 3.2 & 21.4 $\pm$ 3.0  & 14.6 $\pm$ 1.7 & 55.6 $\pm$ 2.2  & 23.2 $\pm$ 1.9 & 14.4 $\pm$ 1.4 \\
    \bottomrule
  \end{tabular}
  }
  }
  \caption{Validation results of proof generation on EntailmentBank~\cite{dalvi2021explaining}. Results of GPT-3 and Codex are based on prompting with 7 in-context examples randomly sampled from the training data.}
  \label{table:entailmentbank_main_val}
\end{table*}

\begin{table*}[ht]
  \small
  \centering 
  \begin{tabular}{@{}lllllllllllll@{}}
    \toprule
     Method & \multicolumn{6}{c}{Answer accuracy} & \multicolumn{6}{c}{Proof accuracy} \\
     \cmidrule(r){2-7} \cmidrule(r){8-13} & N/A & 0 & 1 & 2 & 3 & All & N/A & 0 & 1 & 2 & 3 & All \\
    \midrule
    \name & 99.7 & 100.0 & 100.0 & 99.4 & 97.0 & 99.4 & 99.7 & 100 & 100 & 99.4 & 96.1 & 99.3 \\
    \bottomrule
  \end{tabular}
  \caption{Validation results on the D0--D3 subset  of RuleTaker (OWA)~\cite{tafjord2021proofwriter}.}
  \label{table:ruletaker_d03_val}
\end{table*}

\section{Distinguishing Valid and Invalid Hypotheses}
\label{sec:bostrom}

\begin{table}[ht]
  \small
  \centering 
  \begin{tabular}{@{}lll@{}}
    \toprule
     Method & Task 1 & Task 2 \\
    \midrule
    Learned (Goal) + PPM\textsuperscript{$\dagger$} & \underline{82.0 $\pm$ 1.0} & \underline{86.0 $\pm$ 1.0} \\
    EntailmentWriter\textsuperscript{$\dagger$} & 53.0 $\pm$ 2.0 & 65.0 $\pm$ 2.0 \\
    Ours & \textbf{82.4 $\pm$ 0.8} & \textbf{90.9 $\pm$ 1.2} \\
    \bottomrule
  \end{tabular}
  \caption{Area under the ROC curve (AUROC) of distinguishing valid/invalid hypotheses on EntailmentBank~\cite{dalvi2021explaining} validation set. Methods with $\dagger$ are reported by \citet{bostrom2022natural}. All methods are based on T5-large~\cite{raffel2020exploring}.}
  \label{table:bostrom_val}
\end{table}

\begin{figure}[ht]
  \centering
  \includegraphics[width=1.0\linewidth]{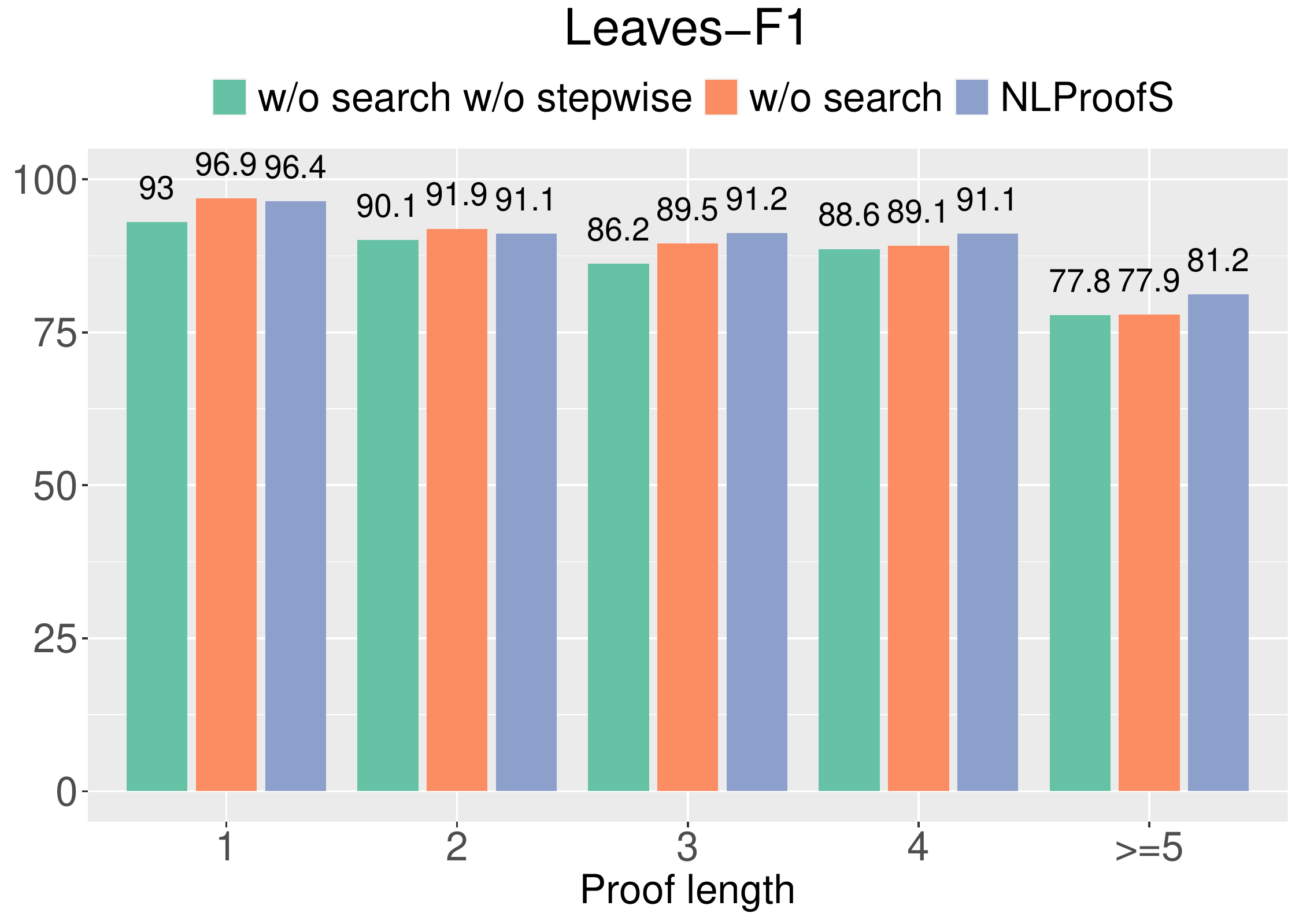}
  \caption{The Leaves-F1 metric of test results on task 2 (distractor) broken down by the length of the ground truth proof.}
  \label{fig:proof_length_leaves_f1}
\end{figure}

\begin{figure}[ht]
  \centering
  \includegraphics[width=1.0\linewidth]{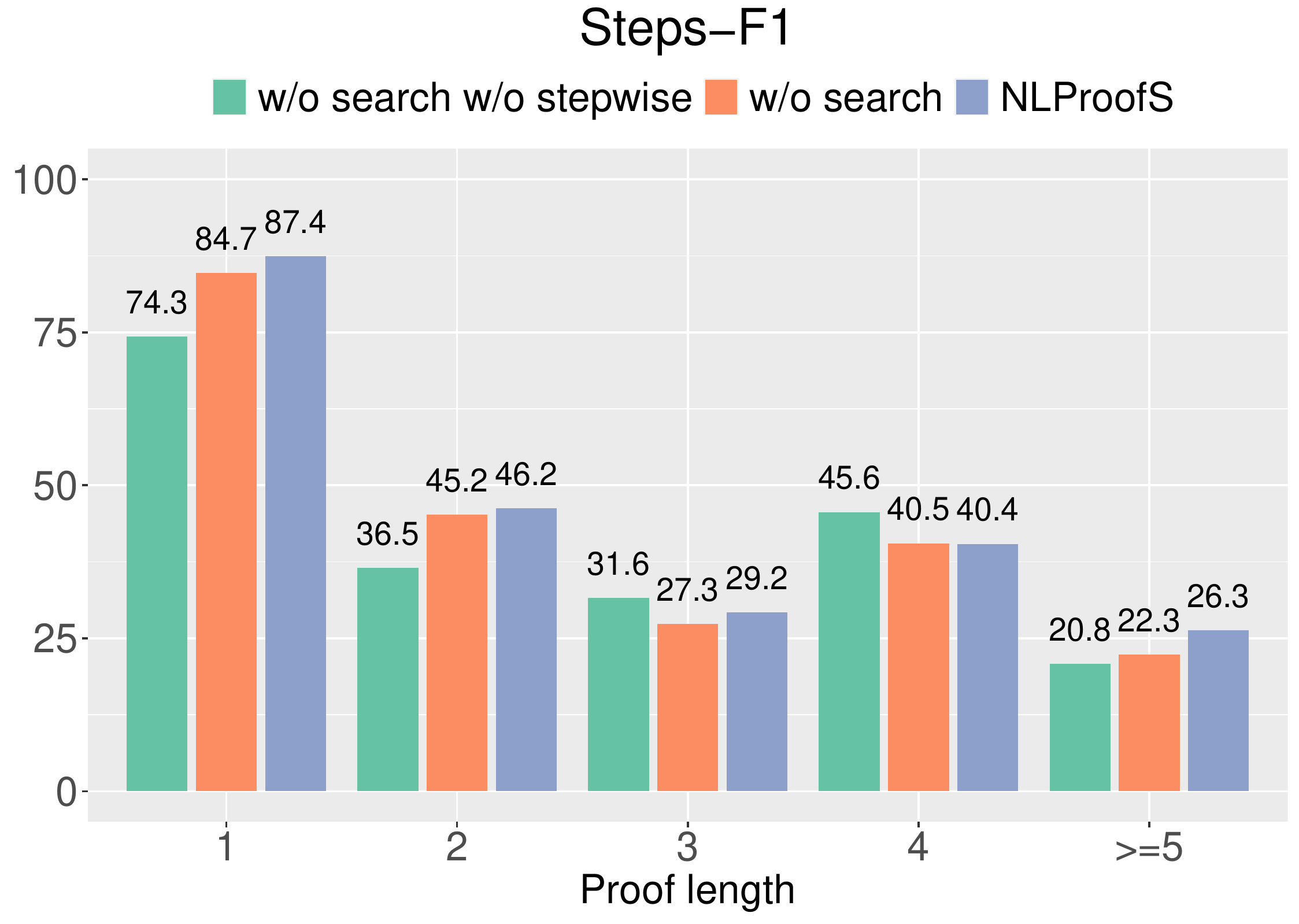}
  \caption{The Steps-F1 metric of test results on task 2 (distractor) broken down by the length of the ground truth proof.
  }
  \label{fig:proof_length_steps_f1}
\end{figure}

\begin{figure}[ht]
  \centering
  \includegraphics[width=1.0\linewidth]{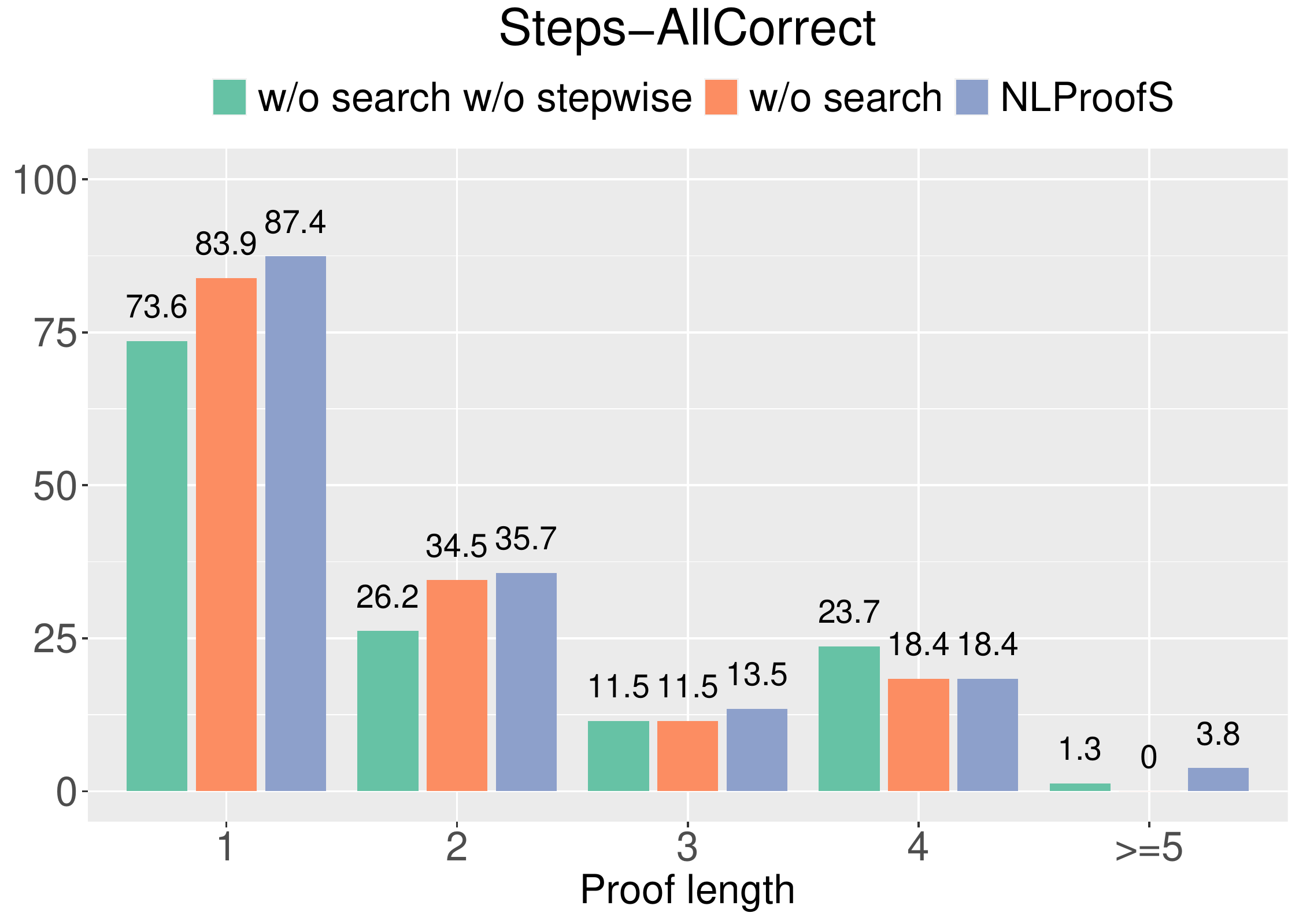}
  \caption{The Steps-AllCorrect metric of test results on task 2 (distractor) broken down by the length of the ground truth proof.
  }
  \label{fig:proof_length_steps_allcorrect}
\end{figure}

\begin{figure}[ht]
  \centering
  \includegraphics[width=1.0\linewidth]{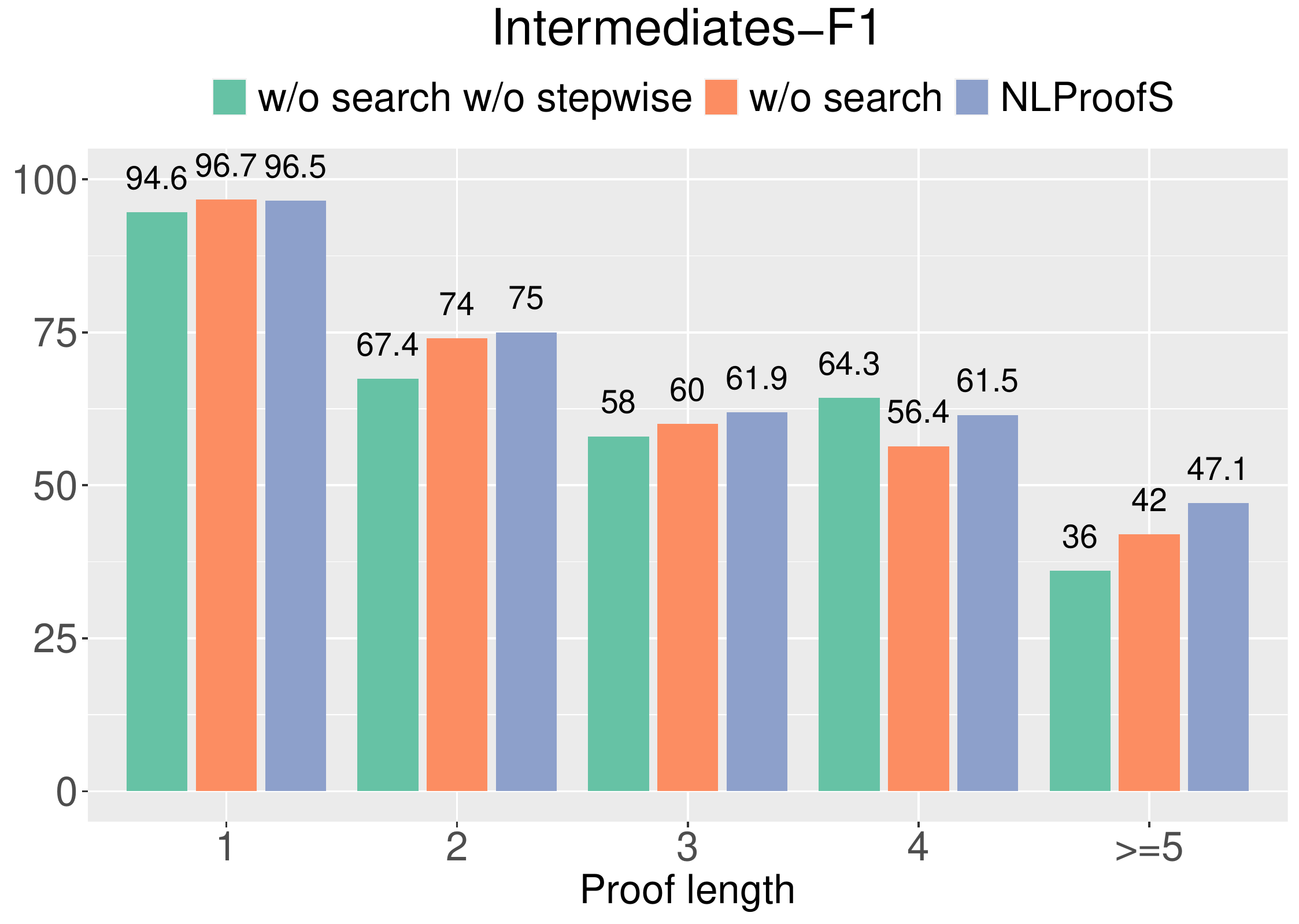}
  \caption{The Intermediates-F1 metric of test results on task 2 (distractor) broken down by the length of the ground truth proof.
  }
  \label{fig:proof_length_intermediates_f1}
\end{figure}

\begin{figure}[ht]
  \centering
  \includegraphics[width=1.0\linewidth]{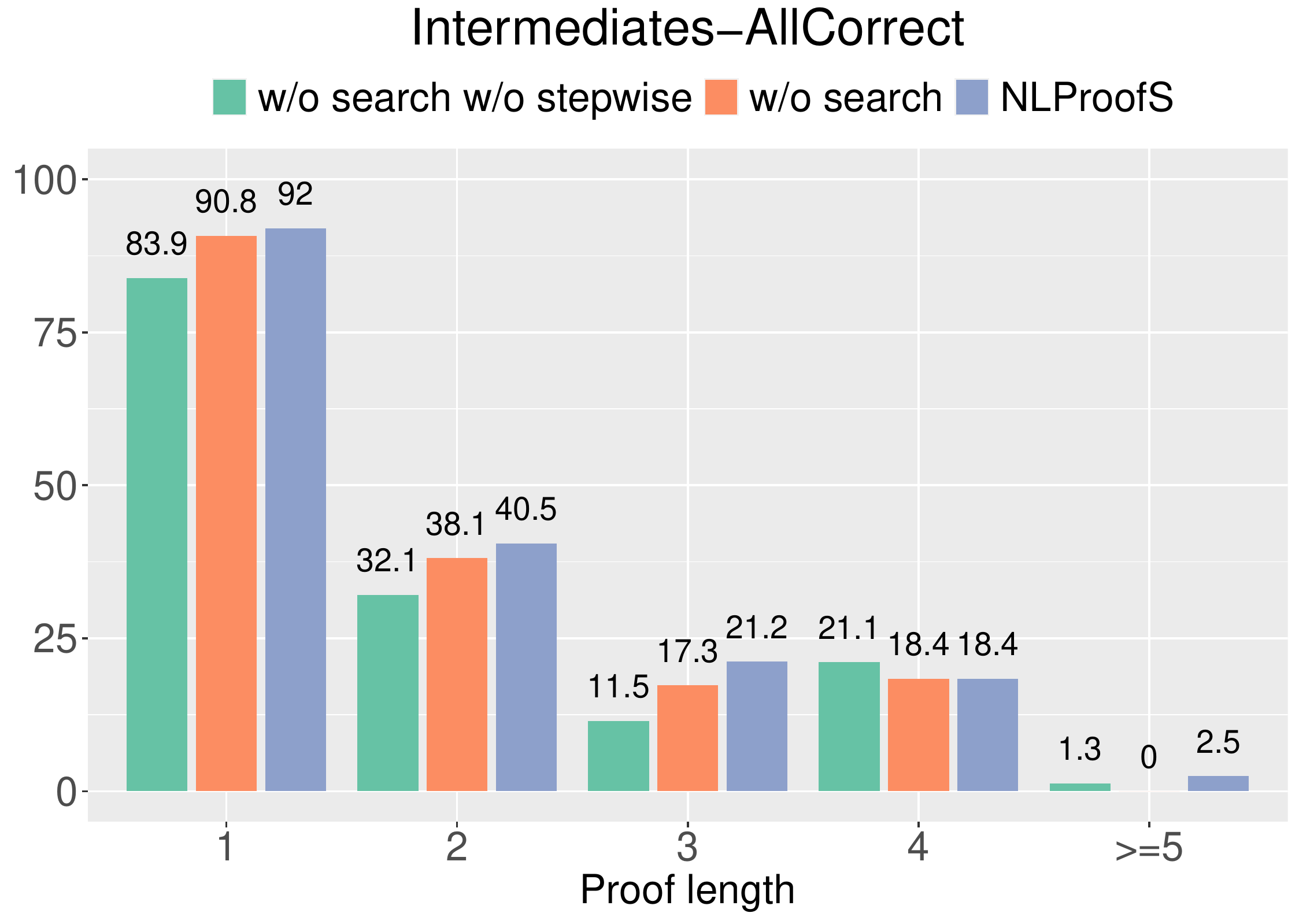}
  \caption{The Intermediates-AllCorrect metric of test results on task 2 (distractor) broken down by the length of the ground truth proof.
  }
  \label{fig:proof_length_intermediates_allcorrect}
\end{figure}

\begin{figure}[ht]
  \centering
  \includegraphics[width=1.0\linewidth]{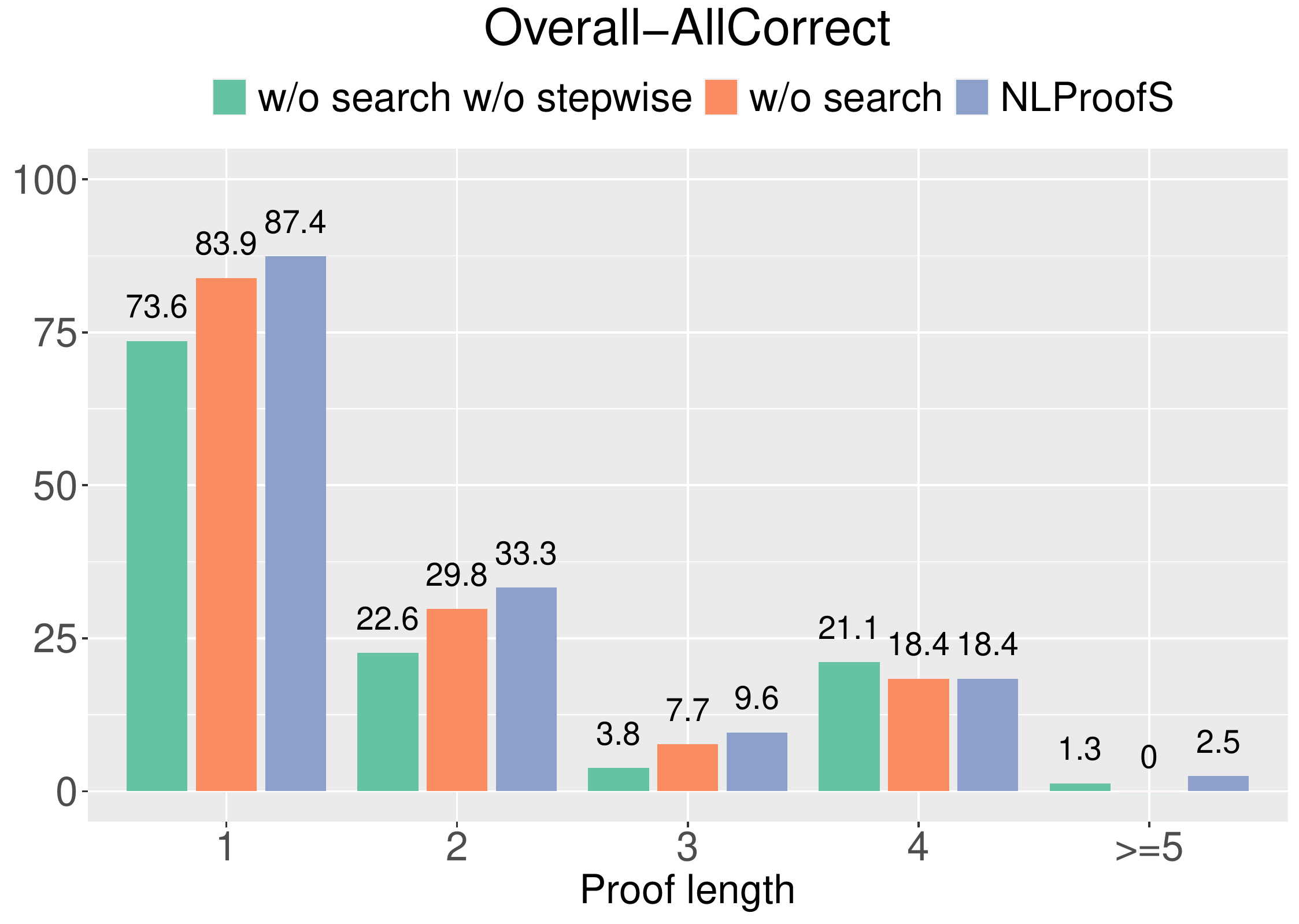}
  \caption{The Overall-AllCorrect metric of test results on task 2 (distractor) broken down by the length of the ground truth proof.
  }
  \label{fig:proof_length_overall_allcorrect}
\end{figure}

We also evaluate on distinguishing valid/invalid hypotheses introduced by \citet{bostrom2022natural}. In this task, the model is given a hypothesis $h$ and supporting facts $C$. But unlike in proof generation, here $h$ can be either valid or invalid w.r.t. $C$. And the model has to classify $h$ as valid/invalid. We use the dataset \citet{bostrom2022natural} constructed from EntailmentBank: Examples with valid hypotheses come directly from EntailmentBank. Examples with invalid hypotheses are constructed by pairing the supporting facts in one example with the hypothesis in another random example.

\name is developed for proof generation, and it has seen only valid hypotheses in training. So we follow \citet{bostrom2022natural} to adapt proof generation systems to this new task: (1) Train the system to generate proofs for valid hypotheses. (2) Apply the system to generate proof scores for both valid and valid hypotheses. (3) Train a linear classifier on top of the scores to predict the validity of hypotheses. It requires the system to be able to produce proof scores. For our method, we use $\texttt{scr}_n(h)$ defined in Eqn~\ref{eqn:proof_score} as the proof score.

Results in Table~\ref{table:bostrom_val} show that our method compares favorably with SCSearch, whereas EntailmentWriter falls behind. The results suggest that proof scores generated by us are more well-calibrated: they are high for valid hypotheses and low for invalid ones. This is largely attributed to our verifier, which prevents the model from hallucinating invalid proofs with confidence.

However, results on this task should be interpreted with caution. First, they do not reflect the performance on proof generation, and SCSearch has not been evaluated on proof generation. Second, none of the methods are explicitly optimized for this task. They see only valid hypotheses during training but are asked to distinguish valid/invalid hypotheses during inference. Third, the particular dataset constructed by \citet{bostrom2022natural} is too easy. An invalid hypothesis has very little lexical overlap with the supporting facts, which can be used as a cue for classifying hypotheses accurately. As a result, a simple RoBERTa baseline directly optimized for classifying the hypothesis can solve this task to almost 100\%.

\section{Additional Experimental Results}
\label{sec:additional_experiments}

\smallsec{Validation results}
Table~\ref{table:entailmentbank_main_val} shows our proof generation results on the validation set of EntailmentBank (Task 2)~\citep{dalvi2021explaining}, corresponding to Table~\ref{table:entailmentbank_main_test}.
Table~\ref{table:ruletaker_d03_val} shows the validation results on RuleTaker (OWA)~\cite{tafjord2021proofwriter}, corresponding to Table~\ref{table:ruletaker_main}. 

\begin{figure*}[ht]
  \centering
  \includegraphics[width=1.0\linewidth]{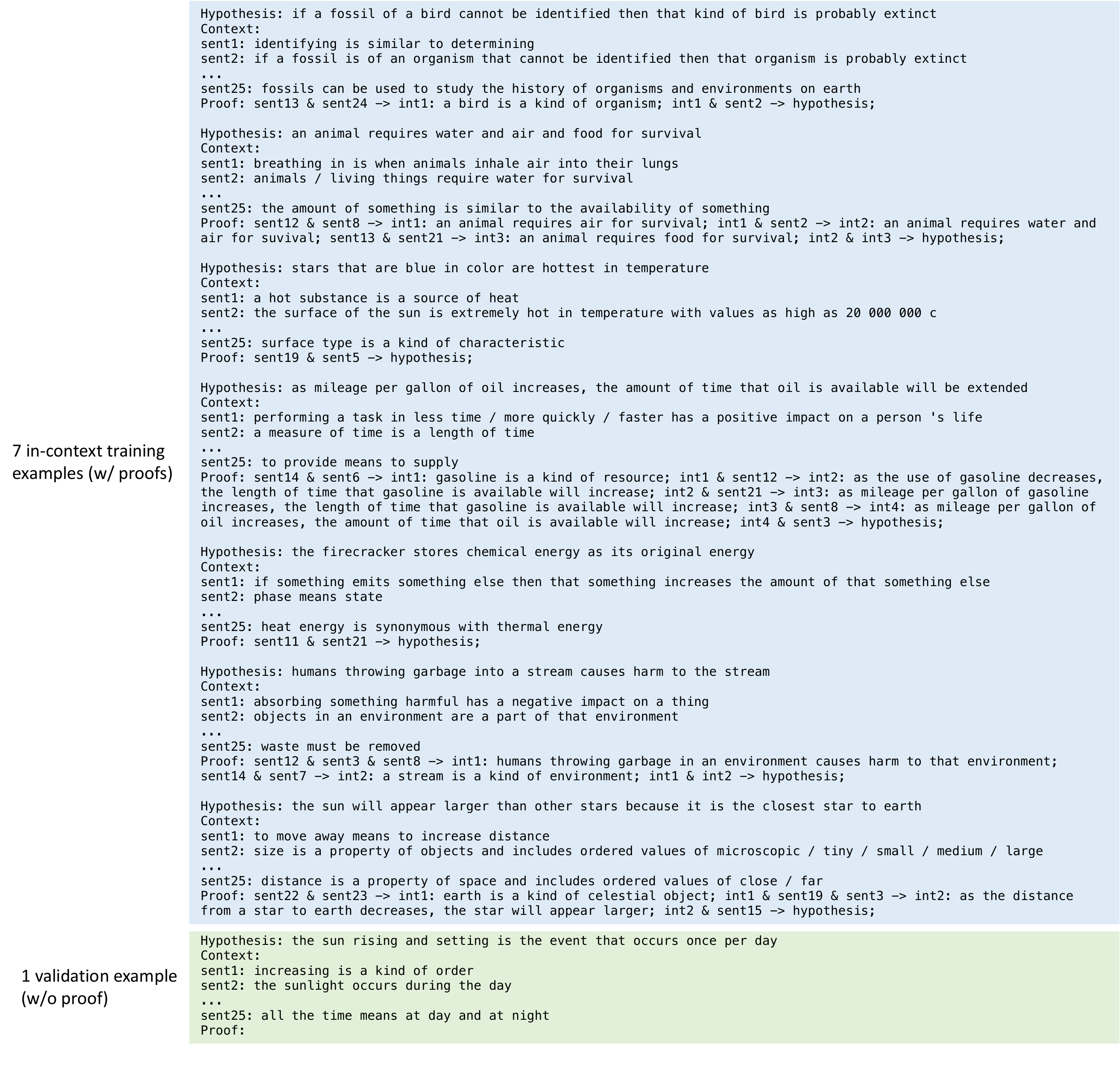}
  \caption{A prompt for GPT-3 and Codex. Each example has 25 supporting facts. We only show 3 for simplicity.}
  \label{fig:prompt}
\end{figure*}

\smallsec{Few-shot prompting with GPT-3 or Codex}
We investigate whether proof generation can be solved out of the box by prompting GPT-3~\citep{brown2020language} or Codex~\citep{chen2021evaluating} with few-shot examples. Fig.~\ref{fig:prompt} shows an example prompt consisting of 7 in-context examples randomly sampled from the training set of EntailmentBank (Task 2), as well as a validation example for which we want to make predictions.

Table~\ref{table:entailmentbank_main_val} includes the results on the full validation set. They were obtained on October 20, 2022 using the model \texttt{text-davinci-002} for GPT-3 and \texttt{code-davinci-002} for Codex. We report the mean and standard deviation from 3 independent runes with different in-context examples in the prompt. GPT-3 and Codex perform substantially worse than other methods, demonstrating that we cannot easily solve proof generation through few-shot prompting. In addition, Codex performs better than GPT-3, which is consistent with the observations in \citet{madaan2022language} though we do not format the output as Python programs.

\smallsec{Test results by different proof length}
Fig.~\ref{fig:proof_length_leaves_f1}, \ref{fig:proof_length_steps_f1}, \ref{fig:proof_length_steps_allcorrect}, \ref{fig:proof_length_intermediates_f1}, \ref{fig:proof_length_intermediates_allcorrect}, and \ref{fig:proof_length_overall_allcorrect} are EntailmentBank (Task 2) test results broken down by proof length (also Fig.~\ref{fig:proof_length_leaves}).

\smallsec{Improving the retriever} 
For Task 3 of EntailmentBank, all methods in Table~\ref{table:entailmentbank_main_test} use the same retrieved supporting facts in \citet{dalvi2021explaining} and focus solely on proof generation. An orthogonal direction is improving the retriever. IRGR~\citep{ribeiro2022entailment} designs a multi-step retriever, which obtains significant improvements on Task 3 (11.8\% on the Overall-AllCorrect metric) but worse results on Task 1 and Task 2 compared to the EntailmentWriter baseline. We do not compare with IRGR, since improving the retriever is orthogonal to our contributions. 

\end{document}